%% file: main.tex
\documentclass{article}

\usepackage[T1]{fontenc}
\usepackage[utf8]{inputenc}
\usepackage{arxiv}
\usepackage{microtype}
\usepackage[numbers]{natbib}
\usepackage{hyperref}
\usepackage{amsmath,amssymb,amsthm}
\usepackage{graphicx}
\usepackage{verbatim}
\usepackage{color}
\usepackage{url}
\usepackage{xcolor}
\IfFileExists{booktabs.sty}{\usepackage{booktabs}}{}
\IfFileExists{siunitx.sty}{\usepackage{siunitx}}{}
\IfFileExists{soul.sty}{\usepackage{soul}\soulregister{\cite}{1}}{}
\IfFileExists{orcidlink.sty}{\usepackage{orcidlink}}{%
  \newcommand{\orcidlink}[1]{\textsuperscript{\tiny ORCID}}}
\providecommand{\toprule}{\hline}
\providecommand{\midrule}{\hline}
\providecommand{\bottomrule}{\hline}

\providecommand{\SI}[2]{#1\,#2}
\providecommand{\second}{s}
\providecommand{\gibi}{Gi}
\providecommand{\byte}{B}

\newtheorem{proposition}{Proposition}
\newtheorem{theorem}{Theorem}
\newtheorem{lemma}{Lemma}
\DeclareMathOperator{\Var}{Var}
\DeclareMathOperator{\E}{\mathbb{E}}
\newcommand{\EVar}{\E\Var}        

\title{Graph Dirichlet Energy for Fitting-Free Bivariate Causal Inference}
\author{Povilas Daniu\v{s}is\orcidlink{0000-0001-5977-827X} \\
Neurotechnology and Vytautas Magnus University \\
\texttt{povilas.daniusis@gmail.com}}
\date{}
\renewcommand{\shorttitle}{Graph Dirichlet Energy for Bivariate Causal Inference}
\hypersetup{pdftitle={Graph Dirichlet Energy for Fitting-Free Bivariate Causal Inference},
pdfauthor={Povilas Daniusis}, hidelinks}
\begin{document}
\maketitle

\begin{abstract}
We propose the Graph Dirichlet Energy Causal Inference (GDECI) criterion for
distinguishing cause from effect in bivariate settings. For a candidate
direction $X\to Y$, we construct a kernel similarity graph from $X$ and
treat $Y$ as a graph signal. The score is the Dirichlet energy of this
signal, induced by the unnormalized graph Laplacian. The reverse score
swaps the variables, and the lower-energy direction is inferred as causal.
The method requires neither regression fitting nor independence testing.
After the empirical probability integral transform, its decisions are
invariant under strictly monotone transformations of either variable, and
the graph operations reduce to convolutions, allowing exact evaluation
in $O(n\log n)$ time and $O(n)$ memory. At the population level, GDECI
compares degree-weighted Bayes squared prediction errors plus a non-negative
smoothness term of the regression function. As the bandwidth tends to
zero, the weights become uniform and the smoothness term vanishes, reducing
the criterion to a RECI-type comparison of expected conditional variances
in copula space. Empirically, we compare GDECI with QCCD, GRCI, CAM, IGCI
(two variants), RESIT, LOCI, QPE-k, and RECI (two variants), in accuracy on
1899 cause--effect pairs and in inference speed. GDECI ranked fifth of
eleven methods in macro-averaged accuracy on 12 synthetic suites ($0.776$).
On the real-world Tuebingen pairs, it ranked third in unweighted accuracy
($0.684$) and fifth in weighted accuracy ($0.678$). It was also the fastest
of the compared methods up to $n\approx10^{4}$. The accompanying code repository:~\url{https://github.com/povidanius/gdeci}.

\end{abstract}

\section{Introduction}

Given an i.i.d.\ sample of $(X,Y)\sim P_{X,Y}$, deciding whether $X\to Y$ or
$Y\to X$ is an important problem in causal discovery, with applications
wherever interventions are costly or infeasible~\cite{Mooij2016}. From the
observational joint distribution alone, the causal direction is in general
not identifiable without additional assumptions. Bivariate causal-discovery
methods therefore exploit statistical or structural properties that hold in
the causal direction but generally fail when the roles of cause and effect
are reversed.

A broad class of methods obtains such asymmetries by restricting the
data-generating mechanism. For instance, additive-noise
models~\cite{Hoyer2009} represent the effect as a function of the cause
plus an independent disturbance, $Y=f(X)+N$ with $N$ independent of $X$.
RESIT~\cite{Peters2014} fits a nonlinear regression in each direction and
tests whether the residuals are independent of the regressor, while
CAM~\cite{Buhlmann2014} compares the Gaussian likelihoods of additive
models fitted in the two directions. Location--scale (heteroscedastic) noise models generalize this to $Y=f(X)+g(X)N$, so that the noise scale may depend on the cause and the noise is recovered as the standardized residual $(Y-f(X))/g(X)$. GRCI~\cite{StroblLasko2023}
estimates the conditional mean and scale in each direction and selects the
direction whose standardized residual is less dependent on the candidate
cause, measured by mutual information. LOCI~\cite{immer2023loci} fits $f$
and $g$ with neural networks and selects the direction with either the
higher likelihood or the weaker HSIC dependence between the standardized
residual and the candidate cause. These restrictions make the direction
identifiable because, apart from non-generic cases, no model of the same
form holds in the reverse
direction~\cite{Hoyer2009,Peters2014,StroblLasko2023,immer2023loci}.
QPE-k~\cite{chen2026qpe} restricts the conditional distribution
through the quantile partial effect, defined by differentiating the
conditional quantile function with respect to the cause at a fixed
quantile level. For each value of the cause, this quantity, viewed as a function of the response value, is assumed to be a linear combination of finitely many known basis functions. This assumption contains additive and heteroscedastic noise models as
special cases.

Other methods derive such asymmetries from the principle of independence of cause and mechanism (ICM)~\cite{JanzingSchoelkopf2010},
which they formalize in different ways. For instance, IGCI relies on the fact that, if the
cause distribution and the data-generating mechanism are chosen
independently, irregularities of the latter are generally reflected in the
effect distribution, resulting in an asymmetry in relative
entropy~\cite{Daniusis2010,Janzing2012}. QCCD approaches the independence
of cause and mechanism through minimum description length, comparing the
quantile scores of nonparametric quantile regressions in the two directions
as proxies for code length~\cite{tagasovska2020qccd}. RECI compares the least-squares errors of regressing each variable on the other. Under an independence postulate, the error is smaller in the causal direction for equally scaled variables, invertible nonlinear mechanisms, and sufficiently small noise~\cite{Bloebaum2018,Bloebaum2019}.

In this paper, we investigate whether asymmetries in graph Dirichlet energy can reveal the causal direction between two variables. We treat
one variable as a signal on a similarity graph built from the other and
measure its Dirichlet energy induced by the unnormalized graph
Laplacian. Although graph Dirichlet energy has been widely used in machine
learning and graph signal
processing~\cite{BelkinNiyogi2003,vonLuxburg2007,Shuman2013}, to the best
of our knowledge, this graph-based framework was not investigated from the perspective of
bivariate causal asymmetry. The main contributions are:

\begin{itemize}
\item \textbf{Theoretical and methodological contributions.} We propose a fitting-free GDECI criterion for bivariate causal inference, whose
  decisions are invariant under strictly monotone transformations of either variable, and show that its
  population form decomposes into a weighted Bayes squared prediction error and a
  smoothness term. As the bandwidth tends to zero, the smoothness term vanishes, so that GDECI reduces to a RECI-type comparison in copula space. We give a sufficient condition under which GDECI selects the causal direction for non-injective mechanisms, and show that the smoothness term is of smaller order in the
  causal direction for monotone additive and heteroscedastic mechanisms with
  bounded noise. We also derive a fast algorithm to compute the GDECI criterion in $O(n\log n)$ time and $O(n)$ memory.
\item \textbf{Empirical study.} We benchmark GDECI against QCCD, GRCI, CAM, IGCI, RESIT, LOCI, QPE-k, and RECI methods on 1899 cause--effect pairs and report its empirical profile for both accuracy and inference time.
\end{itemize}

We introduce the proposed criterion in Section~\ref{sec:method}, discuss its theoretical
properties in Section~\ref{sec:theory}, and present an efficient algorithm in
Section~\ref{sec:fft}. Section~\ref{sec:experiments} is devoted to the experiments.
Finally, we discuss limitations in Section~\ref{sec:limitations} and conclude our paper in Section~\ref{sec:conclusion}. Proofs and other technical information are provided in Appendices~\ref{app:invariance_appendix},~\ref{app:prop1},~\ref{app:smoothness}, ~\ref{app:fast}, and~\ref{app:auxilairy_info}.

\section{Method}\label{sec:method}

Let $X,Y$ have continuous distribution functions $F_X,F_Y$ and set $U=F_X(X)$,
$V=F_Y(Y)$, so both marginals are uniform on $[0,1]$ and the directional signal is not
confounded by differing one-dimensional marginals. For i.i.d.\ copies
$(U_i,V_i)_{i=1}^n$ and a bandwidth $\sigma>0$ define
\[
\begin{aligned}
k_\sigma(a,b)&=\exp\!\Big[-\tfrac{(a-b)^2}{2\sigma^2}\Big],
&W^U_{ij}&=k_\sigma(U_i,U_j),\\
d^U_i&=\sum_j W^U_{ij},
&D^U&=\operatorname{diag}(d^U_1,\dots,d^U_n),
\end{aligned}
\]
\[
L^U=D^U-W^U,
\]
and, with $v=(V_1,\dots,V_n)^\top$,
\begin{align}
E_n(U\!\to\!V)&=v^\top L^Uv
=\tfrac12\textstyle\sum_{i,j}W^U_{ij}(V_i-V_j)^2,\nonumber\\
M_n(U)&=\textstyle\sum_{i,j}W^U_{ij},\qquad
\Delta_n(U\!\to\!V)=\dfrac{E_n(U\!\to\!V)}{M_n(U)} .\label{eq:gdeci_score}
\end{align}

The reverse score swaps $U$ and $V$, and the direction with the smaller $\Delta_n$ is
returned. In practice, $U_i$ and $V_i$ are replaced by the normalized empirical
ranks $\operatorname{rank}(X_i)/(n+1)$ and $\operatorname{rank}(Y_i)/(n+1)$, respectively.

\paragraph{Bandwidth.} We parameterize $\sigma$ by a dimensionless multiplier $m>0$
relative to the median interpoint distance,
$\widehat\sigma(z;m)=m\cdot\operatorname{median}_{i<j}|z_i-z_j|$, recomputed per pair,
direction and preprocessing, so $m$ is the only free parameter. Under marginal equalization the two samples are permutations of the same
rank grid, hence $\widehat\sigma(u;m)=\widehat\sigma(v;m)$.

\paragraph{Invariance.}
Since the orientation of a cause--effect pair should not depend on the measurement scale, invariance under strictly monotone transformations of either variable ensures that changes
of measurement units or scale orientation do not affect the inferred
causal direction. GDECI satisfies this requirement through marginal
equalization and the use of the unnormalized graph Laplacian (Appendix~\ref{app:invariance_appendix}). 
For the normalized Laplacian $L_{\mathrm{norm}}=I-D^{-1/2}WD^{-1/2}$, whose null space is spanned by
$D^{1/2}\mathbf 1$ rather than $\mathbf 1$, invariance under strictly
decreasing transformations generally fails.



\section{Theoretical analysis}\label{sec:theory}

We state the main theoretical results below and defer all proofs to
Appendices~\ref{app:prop1} and~\ref{app:smoothness}. Throughout,
$U=F_X(X)$ and $V=F_Y(Y)$ are uniform on $[0,1]$. Write
$d_\sigma(u)=\int_0^1k_\sigma(u,t)\,dt$,
$\mu_\sigma=\int_0^1d_\sigma(u)\,du$ and $w_\sigma=d_\sigma/\mu_\sigma$
(the \emph{degree weights}). Let $r(u)=\E(V\mid U=u)$,
$s(v)=\E(U\mid V=v)$, $\nu_r(u)=\Var(V\mid U=u)$ and
$\nu_s(v)=\Var(U\mid V=v)$. For $g\in L^2[0,1]$, let
$\mathcal D_\sigma[g]=\frac12\int_0^1\!\int_0^1k_\sigma(u,t)
(g(u)-g(t))^2du\,dt$. Define
\begin{equation}\label{eq:pop}
Q_\sigma(U\!\to\!V):=\E[w_\sigma(U)\nu_r(U)]
+\mathcal D_\sigma[r]/\mu_\sigma,
\end{equation}
and $Q_\sigma(V\!\to\!U)$ likewise, with $(s,\nu_s)$ in place of
$(r,\nu_r)$.

\begin{proposition}[Population form]\label{prop:pop}
For every fixed $\sigma>0$, almost surely as $n\to\infty$,
$\Delta_n(U\!\to\!V)\to Q_\sigma(U\!\to\!V)$ and
$\Delta_n(V\!\to\!U)\to Q_\sigma(V\!\to\!U)${\color{black}, also for the rank-based
scores.} As $\sigma\downarrow0$, $\mathcal D_\sigma[g]/\mu_\sigma\to0$
for every $g\in L^2[0,1]$, $Q_\sigma(U\!\to\!V)\to\EVar(V\mid U)$ and
$Q_\sigma(V\!\to\!U)\to\EVar(U\mid V)$.
\end{proposition}

The weights $w_\sigma$ form a probability density on $[0,1]$ that is
smallest at the endpoints and tends to $1$ in the interior as
$\sigma\downarrow0$ (Lemma~\ref{lem:weights}). 
The following results give conditions under which the population rule
selects the causal direction. For non-injective mechanisms
(Theorem~\ref{thm:noninj}), the conditional-variance term determines
the decision, relying on the fact that the cause cannot be recovered from the effect.
For bounded monotone location--scale (heteroscedastic) mechanisms
(Proposition~\ref{prop:hetero}), the smoothness term is of smaller order
in the causal direction as $\sigma\downarrow0$, but both smoothness
terms enter the score difference only beyond first order in $\sigma$.
At first order, the degree weighting rescales the zero-bandwidth
difference and adds a positive boundary term. Hence, at small bandwidth,
the causal direction is selected whenever its expected conditional
variance is not larger.

\begin{theorem}[Non-injective mechanisms]\label{thm:noninj}
Let $Y_\tau=h(f(X)+\tau N)$, $\tau\ge0$, where $X$ and $N$ are
independent with continuous distributions, $f:\mathbb R\to\mathbb R$ is
Borel, $h:\mathbb R\to\mathbb R$ is continuous and strictly increasing,
and $Y_0$ has a continuous distribution. Assume that $f$ is injective
on no Borel set $A$ with $P(X\in A)=1$, and put $U=F_X(X)$,
$V_\tau=F_{Y_\tau}(Y_\tau)$ and $\gamma=\EVar(U\mid Y_0)$. Then
$\gamma>0$, $\EVar(U\mid V_\tau)\ge\gamma$ for every $\tau\ge0$, and
$\EVar(V_\tau\mid U)\to0$ as $\tau\downarrow0$. Consequently, there
exist $\tau_0,\beta_0>0$ such that
$Q_\sigma(U\!\to\!V_\tau)<Q_\sigma(V_\tau\!\to\!U)$ for all
$0<\tau<\tau_0$ and $0<\sigma<\beta_0$, and for each such
$(\tau,\sigma)$ the rank-based rule with bandwidth $\sigma$ selects
$X\to Y_\tau$ almost surely for all sufficiently large $n$.
\end{theorem}

\begin{proposition}[Bounded location--scale mechanisms]\label{prop:hetero}
Let $Y=f(X)+g(X)N$, where $X$ and $N$ are independent and supported on
$[-1,1]$ and $[-\eta,\eta]$, $\eta>0$, with $C^1$ densities that are
strictly positive on these closed intervals. Let $f,g\in C^2([-1,1])$,
$g>0$ and $f'-\eta|g'|\ge\lambda>0$, and put $E_r=\EVar(V\mid U)$,
$E_s=\EVar(U\mid V)$. Then:
(i) there exist $0<c_1\le C_1<\infty$ and $\sigma_1\in(0,1)$ such that
$\mathcal D_\sigma[r]\le C_1\sigma^3$ and
$c_1\sigma^3\log(1/\sigma)\le\mathcal D_\sigma[s]\le
C_1\sigma^3\log(1/\sigma)$ for $0<\sigma<\sigma_1$,
(ii) $\nu_r$ has a continuous version with $\nu_r(0),\nu_r(1)>0$,
$\nu_s$ has a version with $\nu_s(0^+)=\nu_s(1^-)=0$, and, as
$\sigma\downarrow0$,
\[
\begin{aligned}
Q_\sigma(V\!\to\!U)-{}&Q_\sigma(U\!\to\!V)
=\bigl(1+\sigma\sqrt{2/\pi}\bigr)(E_s-E_r)\\
&+\tfrac{\sigma}{\sqrt{2\pi}}\bigl(\nu_r(0)+\nu_r(1)\bigr)+o(\sigma).
\end{aligned}
\]
Hence, if $E_r\le E_s$, then $Q_\sigma(U\!\to\!V)<Q_\sigma(V\!\to\!U)$
for all sufficiently small $\sigma>0$. Monotone additive mechanisms
$Y=f(X)+N$ with $f'\ge\lambda>0$ are the case $g\equiv1$.
\end{proposition}

\textbf{Relation to regression-error criteria.} By
Proposition~\ref{prop:pop}, the population rule compares degree-weighted
Bayes squared prediction errors plus a smoothness term. As
$\sigma\downarrow0$, it corresponds to a RECI-type
comparison~\cite{Bloebaum2019} with Bayes-optimal regression in copula space.
However, GDECI differs from RECI in three respects. First, through rank
equalization, its decisions are invariant to strictly monotone
transformations of either variable (Section~\ref{sec:method}), whereas RECI
is applied on the original scale. Second, no explicit regression is fitted, since GDECI evaluates $\Delta_n(U\!\to\!V)=v^\top L^Uv/\mathbf 1^\top W^U\mathbf 1$ (see \eqref{eq:gdeci_score}). Third, for
$\sigma>0$, degree weighting and the smoothness term change the population
criterion itself. For the mechanisms of Proposition~\ref{prop:hetero}, the
degree weighting already favours the causal direction at first order in
$\sigma$. Empirically, the positive-bandwidth component improved overall accuracy relative
to an extrapolated estimate of the zero-bandwidth limit, with statistical
significance (Section~\ref{sec:ablation}).

\section{Efficient computation}\label{sec:fft}

Evaluating $E_n$ and $M_n$ directly requires all $n^2$ kernel weights and
hence $O(n^2)$ time. After marginal equalization, however, both scores
can be computed exactly in $O(n\log n)$ time. Let
$u_i=\operatorname{rank}(X_i)/(n+1)$ and
$v_i=\operatorname{rank}(Y_i)/(n+1)$, with ties broken at random, so that
$u=(u_1,\dots,u_n)$ and $v=(v_1,\dots,v_n)$ are permutations of the same
grid $\{1,\dots,n\}/(n+1)$. Consequently, both directions share the
bandwidth $\widehat\sigma=h/(n+1)$ with
$h=m\cdot\operatorname{median}_{k<l}|k-l|$, and the kernel weight between
the observations with $U$-ranks $k$ and $l$ equals $\kappa(k-l)$, where
$\kappa(j)=\exp[-j^2/(2h^2)]$, and thus depends only on the lag $k-l$.
Sorting the observations by their $U$-ranks therefore turns $W^U$ into
the Toeplitz matrix $W^\kappa=(\kappa(k-l))_{k,l=1}^n$, and sorting them
instead by their $V$-ranks turns $W^V$ into the same matrix. In the original observation order, $L^U$ and $L^V$ differ, but $z^\top Lz$
is unchanged when $L$ and $z$ are permuted together, so both energies use
the single Laplacian $L^\kappa=\operatorname{diag}(d)-W^\kappa$ with
$d_k=\sum_{l=1}^n\kappa(k-l)$. The two scores therefore differ only through
the signals. Here $v^U$ is obtained from $v$ by reordering its components by
increasing $U$-rank, and $u^V$ is obtained from $u$ by reordering its
components by increasing $V$-rank.

\begin{proposition}[Efficient computation]\label{prop:fast_computation}
Under marginal equalization, with ties broken at random,
$M_n(U)=M_n(V)=\sum_k d_k$ and
\[
\begin{aligned}
E_n(U\!\to\!V)&=\textstyle\sum_k d_k(v^U_k)^2
 -\sum_{k,l}\kappa(k-l)\,v^U_kv^U_l,\\
E_n(V\!\to\!U)&=\textstyle\sum_k d_k(u^V_k)^2
 -\sum_{k,l}\kappa(k-l)\,u^V_ku^V_l,
\end{aligned}
\]
so the GDECI rule selects $X\to Y$ exactly when
$E_n(U\!\to\!V)<E_n(V\!\to\!U)$. Each double sum equals $z^\top W^\kappa z$
for $z\in\{v^U,u^V\}$, where $W^\kappa z$ is a discrete convolution of $z$ with
$(\kappa(j))_{|j|<n}$, so both energies can be evaluated by fast Fourier
transforms in $O(n\log n)$ time and $O(n)$ memory. 
\end{proposition}

The proof is given in Appendix~\ref{app:fast}.

\section{Experiments}\label{sec:experiments}

\subsection{Methods}

We compare GDECI with a set of established methods for bivariate causal
discovery. QCCD~\cite{tagasovska2020qccd} compares quantile-regression
scores motivated by the minimum description length principle.
GRCI~\cite{StroblLasko2023} fits heteroscedastic noise models and
assesses dependence between the inferred noise and the candidate cause.
CAM~\cite{Buhlmann2014} scores causal additive models using a
Gaussian likelihood, whereas RESIT~\cite{Peters2014} combines
nonlinear regression with residual independence testing.
IGCI~\cite{Daniusis2010,Janzing2012} exploits an
assumption of independence of cause distribution
and the causal mechanism. We evaluate variants with uniform and Gaussian
reference measures, denoted by IGCI and $\mathrm{IGCI}_{\mathrm{G}}$,
respectively.
LOCI~\cite{immer2023loci} fits location--scale noise models in both
directions. We use its neural estimator with standardized inputs, 5000
training epochs, and HSIC-based residual independence testing.
RECI~\cite{Bloebaum2018} selects the direction with the smaller
regression error. We evaluate two of its configurations. The first is the default Causal Discovery
Toolbox (CDT)~\cite{Kalainathan2020} implementation, which is included in recent
benchmarks~\cite{chen2026qpe}. The empirical performance of RECI depends
on the regression function class, preprocessing, etc.~\cite{Bloebaum2018}.
With its default degree of three, the toolbox zeroes the linear and
quadratic columns of the design matrix and therefore fits
$y = a+bx^{3}$, a member of the MON family considered
by~\cite{Bloebaum2018}. We additionally evaluate RECI$_{\mathrm{poly}}$, which retains the
toolbox's directional decision rule and $[0,1]$ min--max rescaling,
but uses the polynomial family
$y = \sum_{i=0}^{k}a_i x^{i}$, with
$k \in \{1,\ldots,9\}$, considered by~\cite{Bloebaum2018}.
Finally, QPE-k~\cite{chen2026qpe} estimates quantile partial effects using
kernels and compares how well they satisfy a prescribed basis
representation through a least-squares test.

For the accuracy comparisons, QCCD, GRCI, CAM, IGCI, and RESIT variants
use the released per-pair outputs distributed with the LOCI
repository~\cite{immer2023loci}, rescored on our evaluation pairs.
LOCI, RECI, and QPE-k were run locally.
The GDECI variants and the marginal-only \textsc{EdgeMass} control are
defined in Section~\ref{sec:scores}.

\subsection{Datasets}

We evaluate on 13 established bivariate cause--effect benchmark suites, using
all observations in the versions distributed with the LOCI
repository\footnote{https://github.com/aleximmer/loci/} ~\cite{immer2023loci}.

The four SIM suites originate from Mooij et
al.~\cite{Mooij2016}. Each contains 100 cause--effect pairs with 1000
observations per pair. SIM is the default simulated setting,
SIM-c adds one confounder, SIM-ln uses lower noise levels,
and SIM-G uses a cause distribution that is approximately Gaussian.
Tagasovska et al.~\cite{tagasovska2020qccd} introduced five further synthetic
suites, again with 100 pairs and 1000 observations per pair:
AN and AN-s use additive-noise mechanisms, while
LS and LS-s use location--scale mechanisms, and the ``s''
variants use sigmoidal mechanisms. MN-U uses a sigmoidal mechanism
with multiplicative uniform noise.

The Net, Multi, and Cha rows correspond to the
CE-Net, CE-Multi, and CE-Cha benchmarks used by Goudet et
al.~\cite{goudet2018functional}. In the LOCI distribution each contains 300
pairs. CE-Net uses randomly initialized neural-network mechanisms with varied
cause distributions, whereas CE-Multi uses linear or polynomial mechanisms
with pre- or post-additive and multiplicative noise. CE-Cha is the
continuous-variable subset derived from the ChaLearn Cause--Effect Pairs
Challenge~\cite{guyon2019causeeffect}. Unlike the other synthetic suites, it includes both real and artificial data.

Finally, Tuebingen is the real-world benchmark of
Mooij et al.~\cite{Mooij2016}. We follow the commonly used univariate
continuous protocol and evaluate the 99 eligible pairs~\cite{tagasovska2020qccd}, reporting both unweighted accuracy and accuracy with the supplied pair weights.

\subsection{GDECI scores evaluated}\label{sec:scores}

To assess the effect of marginal equalization, we compare GDECI with
a standardized counterpart that omits the empirical probability integral
transform (PIT) and retains marginal shape information
(Table~\ref{tab:gdeci_scores}). The PIT map $T_{\rm unif}$ assigns normalized
ranks in $\{1,\dots,n\}/(n+1)$, whereas $T_{\rm std}$ centres each
variable and scales it to unit variance. For $x=T(X_1,\dots,X_n)$,
$y=T(Y_1,\dots,Y_n)$ and $\sigma_x=\widehat\sigma(x;m)$, define

\[
\begin{aligned}
W^x_{ij}&=k_{\sigma_x}(x_i,x_j),
&E(x\!\to\!y)&=y^\top L^x y,\\
M(x)&=\sum_{i,j}W^x_{ij},
&L^x&=D^x-W^x,
\end{aligned}
\]
where $D^x$ is the degree matrix, as in Section~\ref{sec:method}.
Reverse scores swap $x$ and $y$; each rule selects the direction
with the smaller score. Bandwidth multipliers were selected by maximizing accuracy in
$300$ cause-effect pairs, generated independently via CDT, and disjoint from every benchmark reported here.

\begin{table}[t]
\centering
\small
\caption{Scoring rules and preprocessing transformations.}
\label{tab:gdeci_scores}
\begin{tabular}{lll}
\toprule
Name & Score & Preprocessing\\
\midrule
$\textsc{GDECI}^{\rm std}_{\rm avg}$
& $E(x\!\to\!y)/M(x)$ & $T_{\rm std}$\\
$\textsc{GDECI}$
& $E(x\!\to\!y)/M(x)$ & $T_{\rm unif}$\\
\textsc{EdgeMass}
& $M(x)$ & $T_{\rm std}$\\
\bottomrule
\end{tabular}
\end{table}

Under $T_{\rm unif}$, both samples are permutations of the same grid,
so $M(x)=M(y)$ and comparing scores reduces to comparing energies.
Under $T_{\rm std}$, the edge masses generally differ, so the
normalization does not cancel. \textsc{EdgeMass} compares the edge
masses alone and therefore uses only marginal information.

\subsection{Benchmark results}
\label{sec:benchmark_results}

We benchmark both accuracy and computational cost. The resulting empirical
profile is summarised in Figure~\ref{fig:empirical_profile}, with inference
time in panel~(a) and accuracy on the synthetic suites and the Tuebingen pairs
in panels~(b) and~(c), respectively. Panel~(d) plots each method's accuracy
against its inference time and shows that no method is both faster and more
accurate than GDECI, on either the synthetic suites or the Tuebingen pairs. 
Per-suite accuracies with 95\% Wilson intervals are
reported in Table~\ref{tab:main}. Throughout, all bandwidths were
frozen before evaluation as described in Section~\ref{sec:scores}.

\paragraph{Accuracy.} Among the investigated methods, GDECI ranked fifth in macro
accuracy on the 12 synthetic suites ($0.776$,
Figure~\ref{fig:empirical_profile}(b)), behind LOCI, QPE-k, GRCI and QCCD.
It was most effective on the additive and location-scale suites and weakest
on Cha, SIM-G and Multi. On the real-world Tuebingen pairs it ranked third
in unweighted accuracy ($0.684$, Figure~\ref{fig:empirical_profile}(c)),
behind GRCI and QCCD, and fifth in weighted accuracy ($0.678$), behind GRCI,
QCCD, RECI and IGCI, although the confidence intervals on these 99 pairs
overlap for most methods. The standardized variant $\textsc{GDECI}^{\rm std}_{\rm avg}$ reached only $0.586$ on the Tuebingen pairs, below \textsc{GDECI}'s $0.684$ (paired
permutation test, $p = 0.041$), whereas on the 12 synthetic suites the two
did not differ significantly ($0.762$ vs.\ $0.776$, paired
permutation test's $p = 0.27$).

\paragraph{RECI configurations.} The two RECI configurations differ substantially. The CDT default
reproduces the monomial rows of~\cite{Bloebaum2018,Bloebaum2019}: on
SIM, SIM-c, SIM-ln and SIM-G it attains $0.44$, $0.53$, $0.44$ and
$0.39$, against the reported monomial ranges $0.42$--$0.46$,
$0.54$--$0.63$, $0.37$--$0.45$ and $0.43$--$0.47$. RECI$_{\mathrm{poly}}$
recovers the accuracies reported there for polynomial regressors of
degree at least two: $0.57$, $0.66$, $0.79$ and $0.59$ on the same four
suites, against the reported ranges $0.52$--$0.57$, $0.65$--$0.68$,
$0.69$--$0.82$ and $0.54$--$0.60$, and $0.656$ on the weighted
Tuebingen pairs against the reported $0.653$--$0.686$ on this benchmark.
Over the twelve non-Tuebingen suites RECI$_{\mathrm{poly}}$ raises the
macro mean from $0.428$ to $0.607$, improving nine suites and worsening
two (McNemar $p<10^{-5}$ over the $1800$ pairs).

\paragraph{Computational cost.} The fast algorithm followed the predicted $O(n\log n)$ rate (Figure~\ref{fig:empirical_profile}(a)), reproduced the direct evaluation, and processed samples of $n=10^6$ in about 0.6\,s. The direct $O(n^2)$ evaluation exceeded the 120\,s budget at $n=10^5$. At $n=10^3$, the size of most benchmark pairs, GDECI was the cheapest method in the comparison. Only RECI, which requires a single least-squares fit per direction, reached comparable speed at $n\approx10^4$, and beyond $n\approx2\times10^4$ it was faster, by a factor of about three at $n=10^6$. LOCI and RESIT exceeded the time budget at $n=10^4$, GRCI at $n\approx2\times10^4$. Taken together, the methods that were more accurate than GDECI on the $12$ synthetic suites were about two to five orders of magnitude slower at $n=10^3$, whereas the methods of comparable cost were less accurate on these suites.

\paragraph{Marginal information.} The comparatively high \textsc{EdgeMass} and
$\mathrm{IGCI}_{\mathrm{G}}$ accuracies on several synthetic suites show that
parts of these benchmarks can be oriented from marginal information alone.
The two agree closely at the suite level ($r=0.99$ across the 12 synthetic
suites), suggesting that both are
largely driven by the same marginal signal. Related observations appear
in~\cite{NEURIPS2021_e987eff4}, where marginal characteristics of simulated
benchmarks are shown to be informative about the causal orientation.

\input{main_table.tex}

\begin{figure}[!t]
  \centering
  \includegraphics[width=0.95\textwidth]{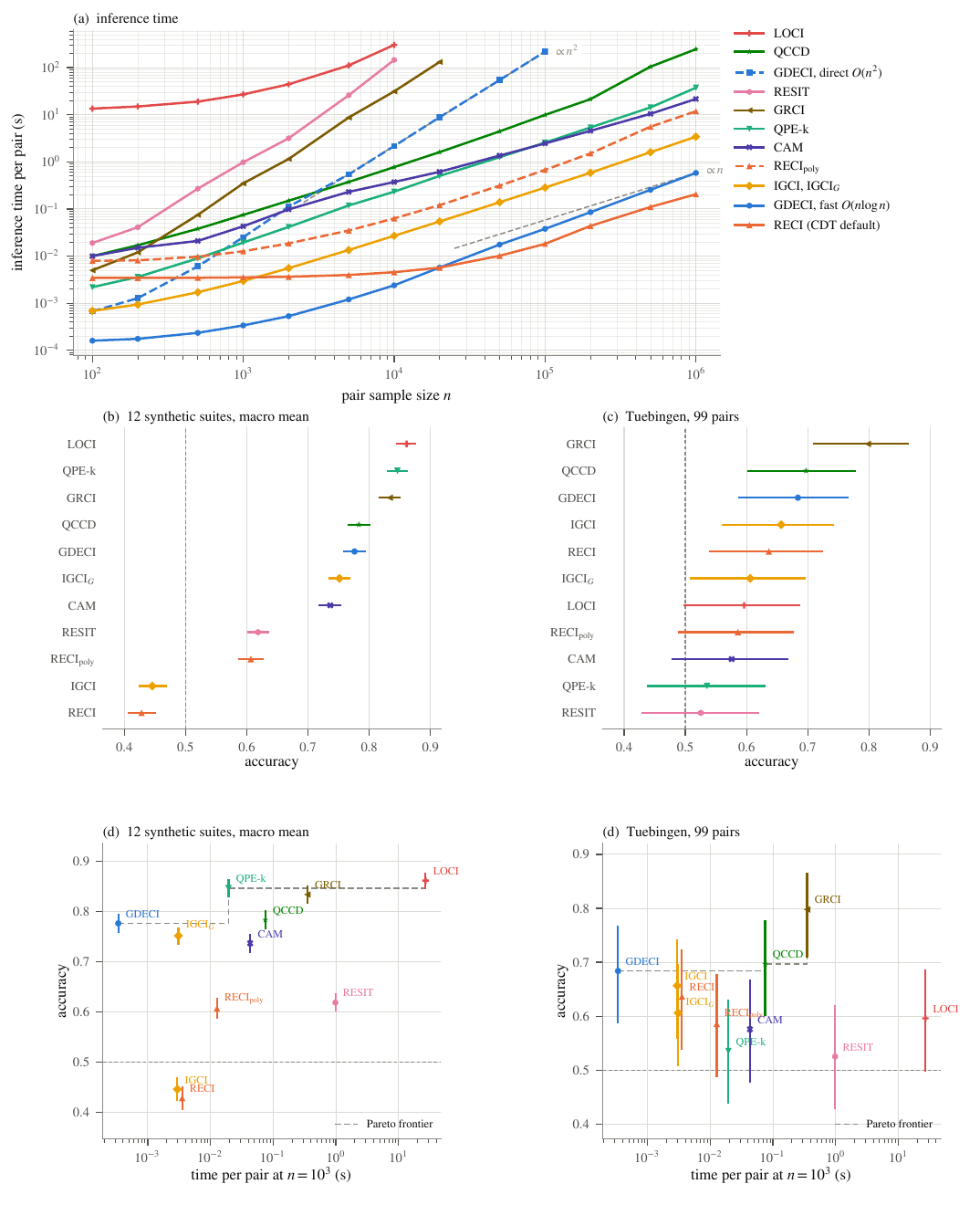}\\[6pt]
  \caption{%
    Empirical profile of GDECI. \textbf{(a):} Time to decide one causal pair, against sample size $n$.
    Each point is the median of seven calls on a single synthetic
    additive-noise pair, using identical inputs for every method.
    Curves terminate where a call first exceeded a \SI{120}{\second} budget or
    exhausted memory. IGCI and IGCI$_G$ share a curve because they differ only
    by an $O(n)$ rescaling. All measurements were taken on one machine
    (Intel Core i7-4790K, \SI{16}{\gibi\byte} RAM). The R
    implementations of QCCD, RESIT, CAM and GRCI were timed within R so that
    interpreter start-up is not charged to them. \textbf{(b)} and \textbf{(c):} accuracies and confidence intervals over benchmarks.
    \textbf{(d):} accuracy from (b) and (c) against time per pair at $n = 10^3$ from (a), with dashed lines marking the Pareto frontier.}
  \label{fig:empirical_profile}
\end{figure}

\subsection{Ablation: does the \texorpdfstring{$\sigma > 0$}{sigma > 0} part carry directional signal?}\label{sec:ablation}

To separate the zero-bandwidth limit of Proposition~\ref{prop:pop} from the
positive-bandwidth part, we score all $1{,}899$ pairs at
$\sigma = s\,\sigma_0$, $s \in \{0.1, 0.2, \dots, 1\}$, where $\sigma_0$ is the
frozen bandwidth of Table~\ref{tab:main}. Causal labels are used only to score
the resulting decisions.

\textbf{Decomposition via a small-bandwidth fit.}
Under the conditions of Lemma~\ref{lem:weights} and
Proposition~\ref{prop:hetero},
$Q_\sigma(U \to V) = \E[w_\sigma(U)\Var(V \mid U)]
+ \mathcal D_\sigma[r]/\mu_\sigma = \EVar(V \mid U) + a\,\sigma + o(\sigma)$.
Since this expansion holds only as $\sigma \to 0$, we fit the GDECI
score~\eqref{eq:gdeci_score} in each direction as $B + a\,\sigma$ over the
five smallest bandwidths ($s \le 0.5$) and keep only the intercept $B$, which
estimates the $\sigma \to 0$ limit $\EVar(V \mid U)$
(Proposition~\ref{prop:pop}). Five points leave three residual degrees of
freedom while staying close enough to $\sigma = 0$ to keep the $o(\sigma)$
remainder small. Having $\Delta_n$ the score at $\sigma_0$, we
compare three decision rules, each choosing the direction with the smaller
value: $B$ (conditional-variance component), $\Delta_n - B$
(positive-bandwidth component) and $\Delta_n$ (full score).

Over all $1{,}899$ pairs, the conditional-variance component alone attained an
accuracy of $0.686$, the positive-bandwidth component alone $0.663$, and the
full score $0.732$ (Figure~\ref{fig:ablation}(a); macro over the 12 synthetic
suites $0.725$, $0.693$ and $0.776$; Tuebingen $0.645$, $0.540$ and $0.684$).
Accuracies are means over ten random tie-breaks. Adding the positive-bandwidth
component corrected $185$ decisions and misdirected $101$, a net gain of $+84$
(McNemar $p = 7.7\times10^{-7}$), and raised pooled
accuracy by $0.045$ (stratified bootstrap 95\% CI $[0.029, 0.062]$). The gain
was positive on $9$ of the $13$ suites, negative on two and zero on two
(one-sided sign test $p = 0.033$). Refitting $B$ on other windows, from
$s \le 0.3$ to $s < 1$, leaves the net gain between $+73$ and $+100$
(Table~\ref{tab:fitwindow}, Appendix~\ref{app:auxilairy_info}).

\textbf{Accuracy along the bandwidth grid.}
As a fit-free check, we score the full GDECI directly at each bandwidth: the
score at $0.1\sigma_0$ is close to the zero-bandwidth limit, and that at
$\sigma_0$ is the full score. Accuracy increased with the bandwidth,
from $0.731$ to $0.776$ macro over the 12 synthetic suites, and from $0.651$ to
$0.684$ on Tuebingen (Figure~\ref{fig:ablation}(b)); both values at $\sigma_0$
match Table~\ref{tab:main}.

\section{Limitations}\label{sec:limitations}

Although GDECI performed competitively across diverse benchmarks,
the provided theoretical guarantees hold only under rather restricted conditions.
These results identify some settings in which the conditional-variance or smoothness term favors the causal direction but do not establish identifiability in general.
In the limit $\sigma \downarrow 0$, taken after $n \to \infty$,
the population rule compares $\mathrm{E}\,\mathrm{Var}(V \mid U)$
with $\mathrm{E}\,\mathrm{Var}(U \mid V)$. Since both marginals
are uniform,
$
\mathrm{E}\,\mathrm{Var}(V \mid U)
=\frac{1-\rho_{U\to V}}{12},
\qquad
\rho_{U\to V}
=\frac{\mathrm{Var}(\mathrm{E}(V|U))}{\mathrm{Var}(V)}.
$
Thus, the limiting rule selects $U\to V$ exactly when
$\rho_{U\to V}>\rho_{V\to U}$, recovering the ordering
of the copula-based directional dependence measures studied by
Sungur~\cite{sungur2005directional}. Sungur studied from the perspective of predictive power, without establishing a causal interpretation. The analogous limiting rule coincides
with RECI~\cite{Bloebaum2019} applied in copula space with
Bayes-optimal regression. 
At $\sigma > 0$, GDECI additionally incorporates degree
weighting and a smoothness term, while remaining fitting-free.
Our ablation shows that positive bandwidth improved macro accuracy
relative to the estimated zero-bandwidth limit, although the gain
was uneven across benchmarks and negative on two of the
12 synthetic suites. Like the bivariate methods reviewed in~\cite{Mooij2016},
GDECI is intended for dependent pairs without unobserved confounding.
In practice, a preliminary independence
test~\cite{ren2024efficiently,daniusis2025measuring}
can help avoid orienting independent pairs, but cannot rule out
unobserved confounding.

\section{Conclusion}\label{sec:conclusion}

We investigated whether asymmetries in graph Dirichlet energy can
reveal the causal direction between two variables. To this end, we
proposed and analysed GDECI, a new fitting-free criterion for bivariate causal
inference that compares the graph Dirichlet energy of each variable,
viewed as a signal on a kernel similarity graph built from the other,
using the unnormalized graph Laplacian. For rank-transformed variables,
the use of the unnormalized Laplacian ensures that its decisions
are invariant under strictly monotone transformations of either
variable, which is not true for the normalized Laplacians.
In population form, GDECI criterion decomposes into a degree-weighted Bayes squared
prediction error and a non-negative smoothness term, both normalized
by the edge mass. As the bandwidth tends to zero, the weights become
uniform and the smoothness term vanishes, recovering a RECI-like
conditional-variance asymmetry in copula space. At positive bandwidth,
the degree-weighted conditional variances and the smoothness terms
of the two directions generally differ, so GDECI can reach a different
decision than a conditional-variance comparison.

\textbf{Theory.} For non-injective mechanisms, we give sufficient conditions under which GDECI selects the causal direction. For additive-noise and heteroscedastic models with bounded supports, in which the cause and noise densities are positive up to the endpoints and
the slope of the effect in the cause is positive, the smoothness term
is of smaller order in the causal direction at sufficiently small bandwidth. Quantitatively, it is at
most of order $\sigma^2$ in the causal direction and of order
$\sigma^2\log(1/\sigma)$ in the anticausal one. The difference arises at the endpoints of the effect's support, where its density vanishes linearly, so that the derivative of the reverse regression function diverges like the inverse square root of the distance to the
copula boundary, whereas the forward regression function is Lipschitz. At first order in the bandwidth, to which the smoothness terms do not contribute, the degree weighting adds a boundary term in favour of the causal direction, because the reverse conditional variance vanishes at the boundary, whereas the forward
one does not. Hence, at small bandwidth, the population rule selects the causal direction whenever its expected conditional variance is not larger. These guarantees take $n\to\infty$ at fixed bandwidth, require the
bandwidth to be sufficiently small, and cover restricted mechanism
classes. They do not establish identifiability in general. On the computational side, marginal equalization reduces the graph operations to convolutions, so GDECI admits an efficient algorithm with
$O(n\log n)$ time and $O(n)$ memory. This makes
it attractive for large samples and for screening many candidate pairs,
for instance before applying slower methods.

\textbf{Empirics.} In the conducted experiments, GDECI showed a favourable trade-off between accuracy and computational cost on the evaluated benchmarks, since no method in the
comparison was both faster and more accurate at $n=10^3$. Among eleven methods, it
ranked fifth in macro accuracy on the 12 synthetic suites, and on
the real-world Tuebingen pairs it ranked third in unweighted and
fifth in weighted accuracy. The four methods that were more accurate on the
synthetic suites were about two to five orders of magnitude slower, and two
of them (LOCI and GRCI) did not reach $n=10^{5}$ within our time and memory budgets.
With the $O(n\log n)$ algorithm, GDECI was the fastest method up to
$n\approx10^{4}$. An extrapolation-based decomposition suggests that the positive-bandwidth
part of the score carries directional information on these pairs:
relative to the conditional-variance component alone, it corrected more
decisions than it misdirected (net $+84$ of $1899$,
$p=7.7\times10^{-7}$), although the gain was not uniform across
benchmarks. Notably, two criteria that infer the direction from the marginals $P_X$ and
$P_Y$ alone, the \textsc{EdgeMass} control and $\mathrm{IGCI}_{\mathrm{G}}$,
reach accuracies of at least $0.92$ on AN, AN-s, LS, LS-s and MN-U, and their
suite-level accuracies correlate at $r=0.99$ across the 12 synthetic suites.
Parts of these benchmarks can thus be oriented from marginal information
alone~\cite{NEURIPS2021_e987eff4}. In contrast, under marginal equalization GDECI
depends on the data only through ranks and hence uses no information from
$P_X$ and $P_Y$. Its accuracy on these suites therefore stems from the dependence between $X$
and $Y$ alone.

Future work includes multivariate extensions and investigation of GDECI-based directionality constraints between latent factors in causal representation learning~\cite{9363924}.

\section{AI Use Statement}\label{sec:declaration_ai}

Claude Opus 5\footnote{https://claude.ai/} was used for writing assistance, mathematical analysis, and code implementation, with all AI-generated material subject to careful manual review, corrections, and modifications.

\section{Acknowledgements}
The author thank Neurotechnology and Vytautas Magnus University for supporting this research. The affiliations reflect the author's institutional associations and do not necessarily indicate institutional participation in this research.

\begingroup
\small
\bibliographystyle{plain}
\bibliography{bibliography}
\endgroup

\clearpage
\section*{Graph Dirichlet Energy for Fitting-Free Bivariate Causal Inference:\
Supplementary Materials}
\appendix

\input{proofs.tex}

\end{document}

%% file: main_table.tex
\begin{table}[!t]\centering\scriptsize\renewcommand{\arraystretch}{0.92}\setlength{\tabcolsep}{4pt}
\caption{Accuracies with $95\%$ Wilson confidence intervals. Panel \textbf{(a)} reflects the published data (LOCI repository), and panel \textbf{(b)} the methods run here together.}
\label{tab:main}

{\footnotesize\textbf{(a) published per-pair outputs}}\\[2pt]
\resizebox{\textwidth}{!}{%
\begin{tabular}{l r rrrrrr}
\toprule
benchmark & $N_{\mathrm{pairs}}$ & QCCD & GRCI & CAM & IGCI & IGCI$_G$ & RESIT\\
\midrule
AN & 100 & $1.000$ {\tiny$[0.96,1.00]$} & $1.000$ {\tiny$[0.96,1.00]$} & $1.000$ {\tiny$[0.96,1.00]$} & $0.200$ {\tiny$[0.13,0.29]$} & $1.000$ {\tiny$[0.96,1.00]$} & $1.000$ {\tiny$[0.96,1.00]$}\\
AN-s & 100 & $0.820$ {\tiny$[0.73,0.88]$} & $0.940$ {\tiny$[0.88,0.97]$} & $1.000$ {\tiny$[0.96,1.00]$} & $0.350$ {\tiny$[0.26,0.45]$} & $1.000$ {\tiny$[0.96,1.00]$} & $1.000$ {\tiny$[0.96,1.00]$}\\
LS & 100 & $1.000$ {\tiny$[0.96,1.00]$} & $0.980$ {\tiny$[0.93,0.99]$} & $1.000$ {\tiny$[0.96,1.00]$} & $0.460$ {\tiny$[0.37,0.56]$} & $0.980$ {\tiny$[0.93,0.99]$} & $0.600$ {\tiny$[0.50,0.69]$}\\
LS-s & 100 & $0.960$ {\tiny$[0.90,0.98]$} & $0.870$ {\tiny$[0.79,0.92]$} & $0.530$ {\tiny$[0.43,0.62]$} & $0.340$ {\tiny$[0.25,0.44]$} & $0.990$ {\tiny$[0.95,1.00]$} & $0.030$ {\tiny$[0.01,0.08]$}\\
MN-U & 100 & $0.990$ {\tiny$[0.95,1.00]$} & $0.880$ {\tiny$[0.80,0.93]$} & $0.860$ {\tiny$[0.78,0.91]$} & $0.110$ {\tiny$[0.06,0.19]$} & $1.000$ {\tiny$[0.96,1.00]$} & $0.050$ {\tiny$[0.02,0.11]$}\\
SIM & 100 & $0.620$ {\tiny$[0.52,0.71]$} & $0.770$ {\tiny$[0.68,0.84]$} & $0.570$ {\tiny$[0.47,0.66]$} & $0.370$ {\tiny$[0.28,0.47]$} & $0.380$ {\tiny$[0.29,0.48]$} & $0.780$ {\tiny$[0.69,0.85]$}\\
SIM-c & 100 & $0.720$ {\tiny$[0.63,0.80]$} & $0.770$ {\tiny$[0.68,0.84]$} & $0.600$ {\tiny$[0.50,0.69]$} & $0.450$ {\tiny$[0.36,0.55]$} & $0.390$ {\tiny$[0.30,0.49]$} & $0.820$ {\tiny$[0.73,0.88]$}\\
SIM-G & 100 & $0.640$ {\tiny$[0.54,0.73]$} & $0.700$ {\tiny$[0.60,0.78]$} & $0.810$ {\tiny$[0.72,0.87]$} & $0.530$ {\tiny$[0.43,0.62]$} & $0.830$ {\tiny$[0.74,0.89]$} & $0.770$ {\tiny$[0.68,0.84]$}\\
SIM-ln & 100 & $0.800$ {\tiny$[0.71,0.87]$} & $0.770$ {\tiny$[0.68,0.84]$} & $0.870$ {\tiny$[0.79,0.92]$} & $0.510$ {\tiny$[0.41,0.61]$} & $0.590$ {\tiny$[0.49,0.68]$} & $0.870$ {\tiny$[0.79,0.92]$}\\
Cha & 300 & $0.537$ {\tiny$[0.48,0.59]$} & $0.700$ {\tiny$[0.65,0.75]$} & $0.467$ {\tiny$[0.41,0.52]$} & $0.550$ {\tiny$[0.49,0.61]$} & $0.580$ {\tiny$[0.52,0.63]$} & $0.343$ {\tiny$[0.29,0.40]$}\\
Multi & 300 & $0.507$ {\tiny$[0.45,0.56]$} & $0.773$ {\tiny$[0.72,0.82]$} & $0.347$ {\tiny$[0.30,0.40]$} & $0.923$ {\tiny$[0.89,0.95]$} & $0.680$ {\tiny$[0.63,0.73]$} & $0.373$ {\tiny$[0.32,0.43]$}\\
Net & 300 & $0.803$ {\tiny$[0.75,0.84]$} & $0.847$ {\tiny$[0.80,0.88]$} & $0.783$ {\tiny$[0.73,0.83]$} & $0.553$ {\tiny$[0.50,0.61]$} & $0.597$ {\tiny$[0.54,0.65]$} & $0.783$ {\tiny$[0.73,0.83]$}\\
Tuebingen & 99 & $0.697$ {\tiny$[0.60,0.78]$} & $0.798$ {\tiny$[0.71,0.87]$} & $0.576$ {\tiny$[0.48,0.67]$} & $0.657$ {\tiny$[0.56,0.74]$} & $0.606$ {\tiny$[0.51,0.70]$} & $0.525$ {\tiny$[0.43,0.62]$}\\
\midrule
\emph{Tuebingen, weighted} & 99 & $0.771$ {\tiny$[0.65,0.86]$} & $0.816$ {\tiny$[0.70,0.90]$} & $0.578$ {\tiny$[0.45,0.70]$} & $0.680$ {\tiny$[0.55,0.79]$} & $0.563$ {\tiny$[0.43,0.68]$} & $0.620$ {\tiny$[0.49,0.74]$}\\
\midrule
\emph{macro mean, 12 non-Tuebingen} & 1800 & $0.783$ {\tiny$[0.76,0.80]$} & $0.833$ {\tiny$[0.81,0.85]$} & $0.736$ {\tiny$[0.72,0.76]$} & $0.446$ {\tiny$[0.42,0.47]$} & $0.751$ {\tiny$[0.73,0.77]$} & $0.618$ {\tiny$[0.60,0.64]$}\\
\bottomrule
\end{tabular}}

\vspace{6pt}

  {\footnotesize\textbf{(b) fully replicated in our study}}\\[2pt]
  \resizebox{\textwidth}{!}{%
  \begin{tabular}{l r rrrr rr r}
  \toprule
  benchmark & $N_{\mathrm{pairs}}$ & \textsc{LOCI} & \textsc{RECI} & RECI$_{\mathrm{poly}}$ & QPE-k & \textsc{GDECI}$^{\mathrm{std}}_{\mathrm{avg}}$ & \textsc{GDECI} & \textsc{EdgeMass}\\
  \midrule
  AN & 100 & $1.000$ {\tiny$[0.96,1.00]$} & $0.180$ {\tiny$[0.12,0.27]$} & $0.920$ {\tiny$[0.85,0.96]$} & $0.990$ {\tiny$[0.95,1.00]$} & $1.000$ {\tiny$[0.96,1.00]$} & $1.000$ {\tiny$[0.96,1.00]$} & $0.940$ {\tiny$[0.88,0.97]$}\\
  AN-s & 100 & $1.000$ {\tiny$[0.96,1.00]$} & $0.350$ {\tiny$[0.26,0.45]$} & $0.610$ {\tiny$[0.51,0.70]$} & $0.880$ {\tiny$[0.80,0.93]$} & $1.000$ {\tiny$[0.96,1.00]$} & $0.830$ {\tiny$[0.74,0.89]$} & $0.960$ {\tiny$[0.90,0.98]$}\\
  LS & 100 & $0.940$ {\tiny$[0.88,0.97]$} & $0.220$ {\tiny$[0.15,0.31]$} & $0.850$ {\tiny$[0.77,0.91]$} & $1.000$ {\tiny$[0.96,1.00]$} & $1.000$ {\tiny$[0.96,1.00]$} & $1.000$ {\tiny$[0.96,1.00]$} & $0.920$ {\tiny$[0.85,0.96]$}\\
  LS-s & 100 & $0.890$ {\tiny$[0.81,0.94]$} & $0.440$ {\tiny$[0.35,0.54]$} & $0.080$ {\tiny$[0.04,0.15]$} & $0.780$ {\tiny$[0.69,0.85]$} & $0.670$ {\tiny$[0.57,0.75]$} & $0.840$ {\tiny$[0.76,0.90]$} & $0.950$ {\tiny$[0.89,0.98]$}\\
  MN-U & 100 & $1.000$ {\tiny$[0.96,1.00]$} & $0.130$ {\tiny$[0.08,0.21]$} & $0.040$ {\tiny$[0.02,0.10]$} & $1.000$ {\tiny$[0.96,1.00]$} & $0.990$ {\tiny$[0.95,1.00]$} & $1.000$ {\tiny$[0.96,1.00]$} & $0.990$ {\tiny$[0.95,1.00]$}\\
  SIM & 100 & $0.780$ {\tiny$[0.69,0.85]$} & $0.440$ {\tiny$[0.35,0.54]$} & $0.570$ {\tiny$[0.47,0.66]$} & $0.830$ {\tiny$[0.74,0.89]$} & $0.510$ {\tiny$[0.41,0.61]$} & $0.650$ {\tiny$[0.55,0.74]$} & $0.350$ {\tiny$[0.26,0.45]$}\\
  SIM-c & 100 & $0.820$ {\tiny$[0.73,0.88]$} & $0.530$ {\tiny$[0.43,0.62]$} & $0.660$ {\tiny$[0.56,0.75]$} & $0.790$ {\tiny$[0.70,0.86]$} & $0.580$ {\tiny$[0.48,0.67]$} & $0.690$ {\tiny$[0.59,0.77]$} & $0.360$ {\tiny$[0.27,0.46]$}\\
  SIM-G & 100 & $0.780$ {\tiny$[0.69,0.85]$} & $0.390$ {\tiny$[0.30,0.49]$} & $0.590$ {\tiny$[0.49,0.68]$} & $0.830$ {\tiny$[0.74,0.89]$} & $0.750$ {\tiny$[0.66,0.82]$} & $0.560$ {\tiny$[0.46,0.65]$} & $0.850$ {\tiny$[0.77,0.91]$}\\
  SIM-ln & 100 & $0.740$ {\tiny$[0.65,0.82]$} & $0.440$ {\tiny$[0.35,0.54]$} & $0.790$ {\tiny$[0.70,0.86]$} & $0.680$ {\tiny$[0.58,0.76]$} & $0.810$ {\tiny$[0.72,0.87]$} & $0.790$ {\tiny$[0.70,0.86]$} & $0.510$ {\tiny$[0.41,0.61]$}\\
  Cha & 300 & $0.727$ {\tiny$[0.67,0.77]$} & $0.560$ {\tiny$[0.50,0.62]$} & $0.583$ {\tiny$[0.53,0.64]$} & $0.597$ {\tiny$[0.54,0.65]$} & $0.437$ {\tiny$[0.38,0.49]$} & $0.499$ {\tiny$[0.44,0.56]$} & $0.583$ {\tiny$[0.53,0.64]$}\\
  Multi & 300 & $0.790$ {\tiny$[0.74,0.83]$} & $0.853$ {\tiny$[0.81,0.89]$} & $0.853$ {\tiny$[0.81,0.89]$} & $0.883$ {\tiny$[0.84,0.91]$} & $0.647$ {\tiny$[0.59,0.70]$} & $0.607$ {\tiny$[0.55,0.66]$} & $0.683$ {\tiny$[0.63,0.73]$}\\
  Net & 300 & $0.867$ {\tiny$[0.82,0.90]$} & $0.603$ {\tiny$[0.55,0.66]$} & $0.733$ {\tiny$[0.68,0.78]$} & $0.893$ {\tiny$[0.85,0.92]$} & $0.750$ {\tiny$[0.70,0.80]$} & $0.847$ {\tiny$[0.80,0.88]$} & $0.557$ {\tiny$[0.50,0.61]$}\\
  Tuebingen & 99 & $0.596$ {\tiny$[0.50,0.69]$} & $0.636$ {\tiny$[0.54,0.72]$} & $0.586$ {\tiny$[0.49,0.68]$} & $0.535$ {\tiny$[0.44,0.63]$} & $0.586$ {\tiny$[0.49,0.68]$} & $0.684$ {\tiny$[0.59,0.77]$} & $0.263$ {\tiny$[0.19,0.36]$}\\
  \midrule
  \emph{Tuebingen, weighted} & 99 & $0.621$ {\tiny$[0.49,0.74]$} & $0.736$ {\tiny$[0.61,0.83]$} & $0.656$ {\tiny$[0.53,0.77]$} & $0.613$ {\tiny$[0.48,0.73]$} & $0.562$ {\tiny$[0.43,0.68]$} & $0.678$ {\tiny$[0.55,0.78]$} & $0.345$ {\tiny$[0.23,0.47]$}\\
  \midrule
  \emph{macro mean, 12 non-Tuebingen} & 1800 & $0.861$ {\tiny$[0.84,0.88]$} & $0.428$ {\tiny$[0.40,0.45]$} & $0.607$ {\tiny$[0.59,0.63]$} & $0.846$ {\tiny$[0.83,0.86]$} & $0.762$ {\tiny$[0.74,0.78]$} & $0.776$ {\tiny$[0.76,0.79]$} & $0.721$ {\tiny$[0.70,0.74]$}\\
  \bottomrule
  \end{tabular}}
\vspace{3pt}
\end{table}

%% file: proofs.tex
\providecommand{\E}{\mathbb{E}}
\providecommand{\Var}{\operatorname{Var}}
\providecommand{\EVar}{\E\Var}
\providecommand{\R}{\mathbb{R}}

\section{Invariance under monotone transformations.}\label{app:invariance_appendix}
Let $u,v$ be the rank vectors of $X,Y$ (no ties). For $z\in\{u,v\}$ put
$W^z_{ij}=k_\sigma(z_i,z_j)$, $L^z=\operatorname{diag}(W^z\mathbf 1)-W^z$ and
$M(z)=\mathbf 1^\top W^z\mathbf 1$, with $\sigma$ fixed or median-based. Since
$u$ and $v$ are permutations of the same grid, $M(u)=M(v)$, and the rule
compares $E(u\!\to\!v)=v^\top L^uv$ with $E(v\!\to\!u)=u^\top L^vu$.
A strictly increasing map leaves ranks unchanged, and a strictly decreasing one
maps $z\mapsto\mathbf 1-z$. Two facts suffice:
\begin{itemize}
\item[(F1)] $L^z$ is symmetric with $L^z\mathbf 1=0$, so
$(\mathbf 1-x)^\top L^z(\mathbf 1-x)=x^\top L^zx$;
\item[(F2)] $|(1-z_i)-(1-z_j)|=|z_i-z_j|$, so $L^{\mathbf 1-z}=L^z$ and
$M(\mathbf 1-z)=M(z)$.
\end{itemize}
\begin{center}
\begin{tabular}{llll}
\hline
$X$ & $Y$ & $E(u\!\to\!v)$ becomes & $E(v\!\to\!u)$ becomes\\
\hline
incr. & incr. & $v^\top L^uv$ & $u^\top L^vu$\\
incr. & decr. & $(\mathbf 1-v)^\top L^u(\mathbf 1-v)\stackrel{\rm F1}{=}v^\top L^uv$
& $u^\top L^{\mathbf 1-v}u\stackrel{\rm F2}{=}u^\top L^vu$\\
decr. & incr. & $v^\top L^{\mathbf 1-u}v\stackrel{\rm F2}{=}v^\top L^uv$
& $(\mathbf 1-u)^\top L^v(\mathbf 1-u)\stackrel{\rm F1}{=}u^\top L^vu$\\
decr. & decr. & $(\mathbf 1-v)^\top L^{\mathbf 1-u}(\mathbf 1-v)\stackrel{\rm F1,F2}{=}v^\top L^uv$
& $(\mathbf 1-u)^\top L^{\mathbf 1-v}(\mathbf 1-u)\stackrel{\rm F1,F2}{=}u^\top L^vu$\\
\hline
\end{tabular}
\end{center}

In all four cases both energies, and hence the decision, are unchanged.
For normalized Dirichlet energy based on $L_{\mathrm{norm}} = I-D^{-1/2}WD^{-1/2}$, property (F1) generally fails.
By Proposition~\ref{prop:pop}, invariance of both empirical scores also implies invariance of their population limits.

\section{Population limit and partial identifiability}\label{app:prop1}

Let $(U_i,V_i)_{i\ge1}$ be i.i.d.\ copies of $(U,V)$ with both marginals
uniform on $[0,1]$. Since $U$
is uniform, expectations of functions of $U$ are integrals over $[0,1]$:
$\mathbb E[g(U)]=\int_0^1 g(u)\,du$. For $\sigma>0$, define
\[
k_\sigma(u,t):=\exp\!\left(-\frac{(u-t)^2}{2\sigma^2}\right),
\qquad
d_\sigma(u):=\int_0^1 k_\sigma(u,t)\,dt,
\qquad
\mu_\sigma:=\int_0^1 d_\sigma(u)\,du.
\]
For a sample of size $n$, write
\[
W_{ij}:=k_\sigma(U_i,U_j),\qquad
d_i:=\sum_{j=1}^n W_{ij},\qquad
D:=\operatorname{diag}(d_1,\ldots,d_n),\qquad
v:=(V_1,\ldots,V_n)^\top,
\]
and set
{\color{black}
\[
\begin{aligned}
L&:=D-W,
& E_n&:=v^\top Lv=\frac12\sum_{i,j=1}^n W_{ij}(V_i-V_j)^2,\\
M_n&:=\sum_{i,j=1}^n W_{ij},
& \Delta_n(U\!\to\!V)&:=E_n/M_n,
\end{aligned}
\]
}%
where $W=(W_{ij})_{i,j=1}^n$, and the
diagonal entries $W_{ii}=1$ are included{\color{black}; they contribute to
$M_n$ but not to $E_n$}. The reverse quantities are obtained
by interchanging $U$ and $V$.
{\color{black}The second expression for $E_n$ follows from $W_{ij}=W_{ji}$ and
$d_i=\sum_jW_{ij}$, since
$\sum_i d_iV_i^2-\sum_{i,j}W_{ij}V_iV_j
=\frac12\sum_{i,j}W_{ij}(V_i^2+V_j^2-2V_iV_j)$. In particular $L$ is symmetric
and $L\mathbf 1=0$.}

Let $F(t)=t$ on $[0,1]$ be the uniform distribution function, and let $F_n$
and $H_n$ be the empirical distribution functions of $(U_i)_{i=1}^n$ and
$(V_i)_{i=1}^n$.

Write $r(u):=\mathbb E[V\mid U=u]$ and $s(t):=\mathbb E[U\mid V=t]$ for the
two regression functions{\color{black}, and
$\nu_r(u):=\E[(V-r(U))^2\mid U=u]$ and $\nu_s(t):=\E[(U-s(V))^2\mid V=t]$ for
the two conditional variances, so that
$\EVar(V\mid U)=\int_0^1\nu_r(u)\,du$ and $0\le\nu_r,\nu_s\le1$}. Let
$\langle g,h\rangle:=\int_0^1 g(u)h(u)\,du$ and
$\|g\|_2^2:=\langle g,g\rangle$, and let $\|\cdot\|_\infty$ denote the
supremum norm. For real-valued $g\in L^2[0,1]$, define the
Dirichlet form and the kernel operator
{\color{black}
\[
\mathcal D_\sigma[g]:=\frac12\int_0^1\!\!\int_0^1 k_\sigma(u,t)
\bigl(g(u)-g(t)\bigr)^2\,du\,dt,
\qquad
(K_\sigma g)(u):=\int_0^1 k_\sigma(u,t)g(t)\,dt.
\]
Expanding the square and using the symmetry of $k_\sigma$ and
$\int_0^1k_\sigma(u,t)\,dt=d_\sigma(u)$,
\begin{equation}\label{eq:proof-form}
\mathcal D_\sigma[g]=\int_0^1 d_\sigma(u)g(u)^2\,du-\langle g,K_\sigma g\rangle,
\end{equation}
which is the population counterpart of $v^\top Lv=\sum_id_iV_i^2-v^\top Wv$.
Since $\mathcal D_\sigma[g]$ depends on $g$ only through the differences
$g(u)-g(t)$, we have $\mathcal D_\sigma[a+bg]=b^2\mathcal D_\sigma[g]$ for all
constants $a,b$.
}

Since $|u-t|\le1$ on $[0,1]^2$, the constant
$c_\sigma:=\exp[-1/(2\sigma^2)]$ satisfies, for every $u,t$,
{\color{black}
\begin{equation}\label{eq:proof-bounds}
c_\sigma\le k_\sigma(u,t)\le1,\qquad
c_\sigma\le d_\sigma(u)\le1,\qquad
c_\sigma\le \mu_\sigma\le1 .
\end{equation}
}


{\color{black}
\paragraph{Degree profile.}
Let $\Phi$ and $\phi$ be the standard Gaussian distribution function and
density, and put
\[
\begin{gathered}
q_\sigma(z):=\frac{e^{-z^2/(2\sigma^2)}}{\sigma\sqrt{2\pi}},
\qquad
\bar\psi_\sigma:=\int_0^1\psi_\sigma(u)\,du,\\
\psi_\sigma(u):=\int_0^1 q_\sigma(u-t)\,dt
=\Phi\Bigl(\frac u\sigma\Bigr)-\Phi\Bigl(\frac{u-1}\sigma\Bigr).
\end{gathered}
\]
Then $k_\sigma(u,t)=\sigma\sqrt{2\pi}\,q_\sigma(u-t)$,
$d_\sigma=\sigma\sqrt{2\pi}\,\psi_\sigma$ and
$\mu_\sigma=\sigma\sqrt{2\pi}\,\bar\psi_\sigma$. We call
\[
w_\sigma:=\frac{d_\sigma}{\mu_\sigma}=\frac{\psi_\sigma}{\bar\psi_\sigma}
\]
the degree weights. With this notation, the scores of
Proposition~\ref{prop:pop} below read
$\E[w_\sigma(U)\nu_r(U)]+\mathcal D_\sigma[r]/\mu_\sigma$ and
$\E[w_\sigma(V)\nu_s(V)]+\mathcal D_\sigma[s]/\mu_\sigma$.

\begin{lemma}[Degree weights]\label{lem:weights}
\emph{(i)} For every $\sigma>0$, $w_\sigma$ is a probability density on
$[0,1]$ with $w_\sigma(1-u)=w_\sigma(u)$. For $0<\sigma\le1$ and every $u\in[0,1]$,
\[
\tfrac14\le\psi_\sigma(u)\le1,\qquad
\tfrac14\le\bar\psi_\sigma\le1,\qquad
\tfrac14\le w_\sigma(u)\le4\text{,}
\]
and $w_\sigma(u)\to1$ for every $u\in(0,1)$ as $\sigma\downarrow0$. Hence
$\int_0^1w_\sigma(u)\nu(u)\,du\to\int_0^1\nu(u)\,du$ for every bounded
measurable $\nu$.

\emph{(ii)} If $\nu$ is bounded and measurable on $[0,1]$ and has one-sided
limits $\nu(0^+)$ and $\nu(1^-)$, then, as $\sigma\downarrow0$,
\[
\int_0^1 w_\sigma(u)\nu(u)\,du
=\int_0^1\nu(u)\,du
+\frac{\sigma}{\sqrt{2\pi}}
\Bigl(2\int_0^1\nu(u)\,du-\nu(0^+)-\nu(1^-)\Bigr)+o(\sigma).
\]
\end{lemma}

\begin{proof}
(i) $\int_0^1w_\sigma(u)\,du=1$ by definition of $\mu_\sigma$, and the
symmetry $\Phi(-z)=1-\Phi(z)$ gives $\psi_\sigma(1-u)=\psi_\sigma(u)$.
Moreover,
$\psi_\sigma'(u)=\sigma^{-1}\bigl[\phi(u/\sigma)-\phi((1-u)/\sigma)\bigr]\ge0$
for $u\le\frac12$, so $\psi_\sigma$ attains its minimum at the endpoints:
$\psi_\sigma\ge\psi_\sigma(0)=\frac12-\Phi(-1/\sigma)\ge\frac12-\Phi(-1)>\frac14$
for $\sigma\le1$. Clearly $\psi_\sigma\le1$, hence
$\frac14\le\bar\psi_\sigma\le1$ and $\frac14\le w_\sigma\le4$. For
$u\in(0,1)$, $\psi_\sigma(u)\to1$, so $\bar\psi_\sigma\to1$ by dominated
convergence and $w_\sigma(u)\to1$. The last claim follows by dominated
convergence.
{\color{black}
(ii) Since $1-\psi_\sigma(u)=\Phi(-u/\sigma)+\Phi(-(1-u)/\sigma)$, and, by
the substitutions $u=\sigma x$ and $u=1-\sigma x$ and dominated
convergence, as $\sigma\downarrow0$,
\[
\begin{aligned}
\frac1\sigma\int_0^1\Phi\Bigl(-\frac u\sigma\Bigr)\nu(u)\,du
&=\int_0^{1/\sigma}\Phi(-x)\,\nu(\sigma x)\,dx
\to\nu(0^+)\int_0^\infty\Phi(-x)\,dx,\\
\frac1\sigma\int_0^1\Phi\Bigl(-\frac{1-u}\sigma\Bigr)\nu(u)\,du
&=\int_0^{1/\sigma}\Phi(-x)\,\nu(1-\sigma x)\,dx
\to\nu(1^-)\int_0^\infty\Phi(-x)\,dx,
\end{aligned}
\]
where $\int_0^\infty\Phi(-x)\,dx=\phi(0)=1/\sqrt{2\pi}$, we get
\[
\int_0^1\psi_\sigma(u)\nu(u)\,du
=\int_0^1\nu(u)\,du-\frac{\sigma}{\sqrt{2\pi}}
\bigl(\nu(0^+)+\nu(1^-)\bigr)+o(\sigma).
\]
The case $\nu\equiv1$ gives $\bar\psi_\sigma=1-\sigma\sqrt{2/\pi}+o(\sigma)$,
and dividing the two expansions gives the claim, because
$\sqrt{2/\pi}=2/\sqrt{2\pi}$. 
}
\end{proof}

\begingroup
\renewcommand{\theproposition}{\ref{prop:pop}}%
\ifdefined\theHproposition
  \renewcommand{\theHproposition}{prop-pop-restated}%
\fi
\begin{proposition}[Population form]
For every fixed $\sigma>0$, almost surely as $n\to\infty$,
\[
\begin{aligned}
\Delta_n(U\!\to\!V)\to Q_\sigma(U\!\to\!V)
&:=\E[w_\sigma(U)\nu_r(U)]+\mathcal D_\sigma[r]/\mu_\sigma,\\
\Delta_n(V\!\to\!U)\to Q_\sigma(V\!\to\!U)
&:=\E[w_\sigma(V)\nu_s(V)]+\mathcal D_\sigma[s]/\mu_\sigma,
\end{aligned}
\]
and the same limits hold for the rank-based scores. Moreover, as
$\sigma\downarrow0$, $\mathcal D_\sigma[g]/\mu_\sigma\to0$ for every
$g\in L^2[0,1]$, $Q_\sigma(U\!\to\!V)\to\EVar(V\mid U)$ and
$Q_\sigma(V\!\to\!U)\to\EVar(U\mid V)$.
\end{proposition}
\addtocounter{proposition}{-1}
\endgroup

{\color{black}
\begin{proof}[Proof of Proposition~\ref{prop:pop}]
Fix $\sigma>0$ and write $Z_i:=(U_i,V_i)$.
The normalized quantities $M_n/n^2$ and $E_n/n^2$ are
$V$-statistics with bounded, symmetric kernels.
The strategy is to compute the expectation of the energy kernel by
conditioning on the graph variables, which splits it into a
degree-weighted conditional variance and a Dirichlet form of the regression
function $r$.

\paragraph{1. Edge mass.}
Including the diagonal entries, the normalized edge mass is a
$V$-statistic:
\begin{equation}\label{eq:proof-mass}
\begin{aligned}
\frac{M_n}{n^2}
&\text{$\displaystyle
=\frac1{n^2}\sum_{i,j=1}^n k_\sigma(U_i,U_j)$}
\xrightarrow{\mathrm{a.s.}}\E[k_\sigma(U,U')]\\
&=\int_0^1\!\!\int_0^1k_\sigma(u,t)\,du\,dt=\mu_\sigma>0,
\end{aligned}
\end{equation}
by the strong law of large numbers for
$V$-statistics with bounded kernels, where $U'$ is an
independent copy of $U$.

\paragraph{2. The energy as a $V$-statistic.}
By the pair-sum representation of $E_n$,
the normalized energy is a $V$-statistic whose diagonal terms vanish:
\begin{equation}\label{eq:proof-energy}
\begin{aligned}
\frac{E_n}{n^2}
&\text{$\displaystyle
=\frac1{n^2}\sum_{i,j=1}^nG_\sigma(Z_i,Z_j)$},\\
G_\sigma((u,v),(t,w))
&:=\tfrac12\,k_\sigma(u,t)\,(v-w)^2\in[0,\tfrac12].
\end{aligned}
\end{equation}
The kernel $G_\sigma$ is bounded and symmetric, so the strong law of large
numbers for $V$-statistics gives
$E_n/n^2\xrightarrow{\mathrm{a.s.}}\E[G_\sigma(Z,Z')]$, where
$Z'=(U',V')$ is an independent copy of $Z=(U,V)$.

\paragraph{3. Conditioning on the graph variables.}
Write
$V-V'=(V-r(U))-(V'-r(U'))+(r(U)-r(U'))$. Since the pairs $(U,V)$ and
$(U',V')$ are independent, $\E[V-r(U)\mid U,U']=\E[V-r(U)\mid U]=0$,
$\E[(V-r(U))^2\mid U,U']=\nu_r(U)$, the analogous identities hold for
$V'-r(U')$, and
$\E[(V-r(U))(V'-r(U'))\mid U,U']
=\E[V-r(U)\mid U]\,\E[V'-r(U')\mid U']=0$. All cross terms involving
the function $r(U)-r(U')$ of $(U,U')$ vanish as well, so
\[
\E\bigl[(V-V')^2\mid U,U'\bigr]
=\nu_r(U)+\nu_r(U')+\bigl(r(U)-r(U')\bigr)^2 .
\]
Since $k_\sigma(U,U')$ is a function of $(U,U')$, the tower property and
the independence and uniformity of $U,U'$ give
\[
\begin{aligned}
\E[G_\sigma(Z,Z')]
&=\frac12\int_0^1\!\!\int_0^1 k_\sigma(u,t)
  \bigl(\nu_r(u)+\nu_r(t)\bigr)\,du\,dt
+\mathcal D_\sigma[r]\\
&=\int_0^1 d_\sigma(u)\nu_r(u)\,du+\mathcal D_\sigma[r],
\end{aligned}
\]
where the symmetry of $k_\sigma$ and
$\int_0^1k_\sigma(u,t)\,dt=d_\sigma(u)$
were used in the last step. The first term equals
$\E[d_\sigma(U)\Var(V\mid U)]$. Thus the variation of $V$ around $r(U)$
enters only through the degree-weighted conditional variance, and the
regression function only through the Dirichlet form $\mathcal D_\sigma[r]$.

\paragraph{4. The score and the reverse direction.}
Dividing by~\eqref{eq:proof-mass} gives
$\Delta_n(U\!\to\!V)=(E_n/n^2)/(M_n/n^2)\to Q_\sigma(U\!\to\!V)$
almost surely, and dividing by $\mu_\sigma$ and using
$w_\sigma=d_\sigma/\mu_\sigma$ gives the form stated in
Proposition~\ref{prop:pop}. Interchanging the roles of $U$ and $V$ gives
the reverse limit with $s$ and $\nu_s$ in place of $r$ and $\nu_r$.
Both marginals are uniform, so both directions share the same
$d_\sigma$ and $\mu_\sigma$.

\paragraph{5. Small bandwidth.}
By Lemma~\ref{lem:Dto0} below, $\mathcal D_\sigma[r]/\mu_\sigma\to0$ and
$\mathcal D_\sigma[s]/\mu_\sigma\to0$. Moreover
$\E[d_\sigma(U)\Var(V\mid U)]/\mu_\sigma=\int_0^1w_\sigma\nu_r
\to\int_0^1\nu_r=\EVar(V\mid U)$ by Lemma~\ref{lem:weights}(i), and likewise
in the reverse direction. The claim for the rank-based scores is proved in
the paragraph \emph{Empirical ranks} below.
\end{proof}

\begin{lemma}[Vanishing Dirichlet term]\label{lem:Dto0}
For every real-valued $g\in L^2[0,1]$,
$\mathcal D_\sigma[g]/\mu_\sigma\to0$ as $\sigma\downarrow0$.
\end{lemma}

\begin{proof}
Let $\bar g$ equal $g$ on $[0,1]$ and zero elsewhere, so that
$\bar g\in L^2(\R)$ and $\|\bar g\|_{L^2(\R)}=\|g\|_2$.
Using $k_\sigma(u,t)=\sigma\sqrt{2\pi}\,q_\sigma(u-t)$ and
$\mu_\sigma=\sigma\sqrt{2\pi}\,\bar\psi_\sigma$, we obtain
\[
\begin{aligned}
\frac{\mathcal D_\sigma[g]}{\mu_\sigma}
&=\frac{1}{2\bar\psi_\sigma}
\int_0^1\int_0^1
q_\sigma(u-t)\bigl(g(u)-g(t)\bigr)^2\,du\,dt\\
&\le\frac{1}{2\bar\psi_\sigma}
\int_{\R}\int_{\R}
q_\sigma(u-t)\bigl(\bar g(u)-\bar g(t)\bigr)^2\,du\,dt\\
&=\frac{1}{2\bar\psi_\sigma}
\int_{\R}\sigma q_\sigma(\sigma z)
\left[\int_{\R}
\bigl(\bar g(t+\sigma z)-\bar g(t)\bigr)^2\,dt\right]dz\\
&=\frac{1}{2\bar\psi_\sigma}
\int_{\R}q_1(z)\,
\bigl\|\bar g(\cdot+\sigma z)-\bar g\bigr\|_{L^2(\R)}^2\,dz.
\end{aligned}
\]
Here we use $u=t+\sigma z$, $du=\sigma\,dz$, Tonelli's theorem, and
$\sigma q_\sigma(\sigma z)=q_1(z)=(2\pi)^{-1/2}e^{-z^2/2}$.

For each fixed $z$, continuity of translations in $L^2(\R)$ gives
$\|\bar g(\cdot+\sigma z)-\bar g\|_{L^2(\R)}\to0$ as $\sigma \downarrow 0$.
Moreover, translation invariance of the $L^2$ norm yields
\[
0\le q_1(z)
\bigl\|\bar g(\cdot+\sigma z)-\bar g\bigr\|_{L^2(\R)}^2
\le 4q_1(z)\|g\|_2^2.
\]
The upper bound is integrable because $\int_{\R}q_1(z)\,dz=1$.
Dominated convergence therefore makes the integral tend to zero.
Finally, $\bar\psi_\sigma\ge\frac14$ for $0<\sigma\le1$
(Lemma~\ref{lem:weights}) bounds the prefactor and proves the claim.
\end{proof}

\paragraph{Consequence for the population ordering.}
{\color{black}By Proposition~\ref{prop:pop}, $Q_\sigma(U\!\to\!V)\to
\EVar(V\mid U)$ and $Q_\sigma(V\!\to\!U)\to\EVar(U\mid V)$ as
$\sigma\downarrow0$.}
If the two expected conditional variances differ, their ordering
thus agrees with the population score ordering for all sufficiently
small positive $\sigma$. This uses $n\to\infty$ at fixed $\sigma$
first. {\color{black}The first-order corrections in
$\sigma$ are described in Lemma~\ref{lem:weights}(ii) and in
Appendix~\ref{app:smoothness}.}

\paragraph{Empirical ranks.}
For the rank grid $\{1,\ldots,n\}/(n+1)$, write
\[
\widehat U_i:=\frac{n}{n+1}F_n(U_i),\qquad
\widehat V_i:=\frac{n}{n+1}H_n(V_i),\qquad
\delta_n:=\max_{i\le n}\max\{
|\widehat U_i-U_i|,|\widehat V_i-V_i|\}.
\]
There are no marginal ties almost surely, and
\[
\delta_n\le
\max\{\|F_n-F\|_\infty,\|H_n-F\|_\infty\}
+\frac1{n+1}\xrightarrow{\mathrm{a.s.}}0.
\]
{\color{black}At fixed $\sigma$, the function
$z\mapsto e^{-z^2/(2\sigma^2)}$
is Lipschitz with constant $\ell_\sigma:=1/(\sigma\sqrt e)$. Replacing
$(U_i,V_i)$ by $(\widehat U_i,\widehat V_i)$ therefore changes each kernel
entry by at most $2\ell_\sigma\delta_n$, and each squared difference
$(V_i-V_j)^2\in[0,1]$ by at most $4\delta_n$. Writing
$E_n^{\mathrm{rank}}$ and $M_n^{\mathrm{rank}}$ for the quantities computed
from $(\widehat U_i,\widehat V_i)$, \eqref{eq:proof-mass} and
\eqref{eq:proof-energy} give
\[
\left|\frac{E_n^{\mathrm{rank}}-E_n}{n^2}\right|
+\left|\frac{M_n^{\mathrm{rank}}-M_n}{n^2}\right|
\le(3\ell_\sigma+2)\,\delta_n\xrightarrow{\mathrm{a.s.}}0.
\]
Since $M_n/n^2\to\mu_\sigma>0$,}
the rank-based scores have the same limits in both directions.


%
%

\paragraph{Conditional variance and best prediction.}
For a bounded random variable $W$ and a random variable $Z$, let
$\rho(z):=\E[W\mid Z=z]$ denote a bounded regression function, so that
$\E[W\mid Z]=\rho(Z)$ a.s., and write $\EVar(W\mid Z):=\E\big(W-\rho(Z)\big)^2$.
For every bounded Borel function $g$,
\begin{equation}\label{eq:bestpred}
\E\big(W-g(Z)\big)^2=\EVar(W\mid Z)+\E\big(\rho(Z)-g(Z)\big)^2 ,
\end{equation}
because in the expansion of $\big((W-\rho(Z))+(\rho(Z)-g(Z))\big)^2$ the cross
term has expectation zero: $\rho(Z)-g(Z)$ is a function of $Z$, and
$\E[W-\rho(Z)\mid Z]=0$. Equivalently, $\rho(Z)$ is the orthogonal projection
of $W$ onto the square-integrable functions of $Z$. Hence $\EVar(W\mid Z)$ is the smallest value of
$\E(W-g(Z))^2$ over bounded Borel $g$, i.e.\ the Bayes squared error of
predicting $W$ from $Z$. For a pair $(U,V)$ with uniform marginals on $[0,1]$
and $\sigma>0$, the population scores of Proposition~\ref{prop:pop} are,
{\color{black}with the degree weights $w_\sigma=d_\sigma/\mu_\sigma$ of
Lemma~\ref{lem:weights},
\begin{equation}\label{eq:Q}
\begin{aligned}
Q_\sigma(U\to V)&=\E\bigl[w_\sigma(U)\bigl(V-r(U)\bigr)^2\bigr]
+\frac{\mathcal D_\sigma[r]}{\mu_\sigma},\\
Q_\sigma(V\to U)&=\E\bigl[w_\sigma(V)\bigl(U-s(V)\bigr)^2\bigr]
+\frac{\mathcal D_\sigma[s]}{\mu_\sigma},
\end{aligned}
\end{equation}
}%
where $r(u)=\E[V\mid U=u]$ and $s(v)=\E[U\mid V=v]$.

{\color{black}
\begin{lemma}[elementary properties of $\mathcal D_\sigma$]\label{lem:Dsig}
For every $\sigma>0$, $g,g'\in L^2[0,1]$ and constants $a,b$,
\[
\begin{gathered}
\mathcal D_\sigma[g]\ge0,\qquad
\mathcal D_\sigma[g+g']\le2\mathcal D_\sigma[g]+2\mathcal D_\sigma[g'],\\
\mathcal D_\sigma[g]\le
2\int_0^1 d_\sigma(u)g(u)^2\,du,\qquad
\mathcal D_\sigma[a+bg]=b^2\mathcal D_\sigma[g].
\end{gathered}
\]
In particular,
$\mathcal D_\sigma[g]/\mu_\sigma
\le2\int_0^1 w_\sigma(u)g(u)^2\,du
\le8\|g\|_2^2$ for $0<\sigma\le1$.
\end{lemma}

\begin{proof}
All integrals are finite, since $k_\sigma\le1$ and $g\in L^2[0,1]$. The
integrand of $\mathcal D_\sigma[g]$ is nonnegative, so
$\mathcal D_\sigma[g]\ge0$. Applying $(\alpha+\beta)^2\le2\alpha^2+2\beta^2$ with
$\alpha=g(u)-g(t)$ and $\beta=g'(u)-g'(t)$ gives the second inequality.
Applying $(\alpha-\beta)^2\le2\alpha^2+2\beta^2$ with $\alpha=g(u)$ and
$\beta=g(t)$, and then using the symmetry of $k_\sigma$ and
$\int_0^1k_\sigma(u,t)\,dt=d_\sigma(u)$,
\[
\mathcal D_\sigma[g]\le\int_0^1\!\!\int_0^1k_\sigma(u,t)\big(g(u)^2+g(t)^2\big)\,du\,dt
=2\int_0^1 d_\sigma(u)g(u)^2\,du .
\]
The identity $\mathcal D_\sigma[a+bg]=b^2\mathcal D_\sigma[g]$ is immediate
from the definition, and the last claim follows from
$d_\sigma=\mu_\sigma w_\sigma$ and $w_\sigma\le4$ (Lemma~\ref{lem:weights}).
\end{proof}
}


\begingroup
\renewcommand{\thetheorem}{\ref{thm:noninj}}%
\ifdefined\theHtheorem
  \renewcommand{\theHtheorem}{thm-noninj-restated}%
\fi
\begin{theorem}[Non-injective mechanisms]
Let $X$ have a continuous distribution and
\begin{equation}\label{eq:model}
Y_\tau=h\big(f(X)+\tau N\big),\qquad \tau\ge0,
\end{equation}
where $f:\R\to\R$ is Borel, $h:\R\to\R$ is continuous and strictly
increasing, and $N$ is independent of $X$ and has a continuous
distribution. Assume that $Y_0$ has a continuous distribution and that
$f$ is injective on no Borel set $A$ with $P(X\in A)=1$. Put
\[
U:=F_X(X),\qquad V_\tau:=F_{Y_\tau}(Y_\tau),\qquad
\gamma:=\EVar(U\mid Y_0).
\]
Then
\begin{align}
&\gamma>0, \label{eq:gamma}\\
&\EVar(U\mid V_\tau)\ \ge\ \gamma\qquad\text{for every }\tau\ge0,
  \label{eq:reverse}\\
&\EVar(V_\tau\mid U)\ \longrightarrow\ 0\qquad\text{as }\tau\downarrow0 .
  \label{eq:forward}
\end{align}
Consequently there exist $\tau_0,\beta_0>0$ such that
\begin{equation}\label{eq:order}
Q_\sigma(U\to V_\tau)\ <\ Q_\sigma(V_\tau\to U)
\qquad\text{for all }0<\tau<\tau_0,\ 0<\sigma<\beta_0,
\end{equation}
and for each such $(\tau,\sigma)$ the rank-based rule with bandwidth
$\sigma$ selects $X\to Y_\tau$ almost surely for all sufficiently
large $n$.
\end{theorem}
\addtocounter{theorem}{-1}
\endgroup

\emph{Conventions.} $(X',N')$ denotes an independent copy of $(X,N)$.
Regression functions of $[0,1]$-valued variables are taken $[0,1]$-valued.
Since $X$ and $N$ are independent, the law of $(X,N)$ is the product
$P_X\otimes P_N$, so by Fubini's theorem
\begin{equation}\label{eq:fubini}
\E\,H(X,N)=\int\E\,H(X,n)\,P_N(dn)=\int\E\,H(x,N)\,P_X(dx)
\end{equation}
for every bounded Borel $H:\R^2\to\R$. The same holds for the independent
pairs $(X,N)$ and $(X',N')$, and for $X$ and $X'$. Equalities between random variables are understood almost surely.

\begin{proof}
Let $(X',N')$ be an independent copy of $(X,N)$, put
$Y'_\tau:=h(f(X')+\tau N')$ and $F_\tau:=F_{Y_\tau}$, and take all
regression functions $[0,1]$-valued. Since $h$ is a continuous bijection of
$\R$ onto the interval $h(\R)$, $Y_0=h(f(X))$ and $f(X)$ are Borel functions
of each other, so by \eqref{eq:bestpred}
$\gamma=\EVar(U\mid Y_0)=\EVar(U\mid f(X))$.

\emph{Marginals.} We check that $U$ and $V_{\tau}$ are uniform, so that Proposition~\ref{prop:pop} applies, and that $U$ determines $X$. For $\tau>0$ and $c\in h(\R)$, independence gives
$P(Y_\tau=c)=\int P\big(N=(h^{-1}(c)-f(x))/\tau\big)\,P_X(dx)=0$, since $N$
has a continuous distribution. For $c\notin h(\R)$ the probability is zero
trivially, and for $\tau=0$ it is zero by assumption. By the probability
integral transform, $U$ and all $V_\tau$ are therefore uniform on $[0,1]$.
Moreover, with the quantile function $F_X^{-1}$ (set to $0$ outside $(0,1)$),
\begin{equation}\label{eq:X-of-U}
X=F_X^{-1}(U),
\end{equation}
since $F_X^{-1}(F_X(x))\ne x$ only if $x$ lies in a nondegenerate interval on
which $F_X$ is constant, and the countably many such intervals carry no
probability.

\emph{Proof of \eqref{eq:gamma}.} Suppose $\gamma=0$. Then
\[
\E\bigl[(U-\E[U\mid f(X)])^2\bigr]=0,
\]
so $U=\E[U\mid f(X)]$ a.s. Thus $U$ is a Borel function of $f(X)$.
Since $X=F_X^{-1}(U)$ a.s., we also have $X=\Psi(f(X))$ a.s.\
for some Borel map $\Psi$. The Borel set $A:=\{x:\Psi(f(x))=x\}$ therefore has full probability.
For $x,x'\in A$, if $f(x)=f(x')$, then
\[
x=\Psi(f(x))=\Psi(f(x'))=x'.
\]
Hence $f$ is injective on $A$, contradicting the hypothesis.

\emph{Proof of \eqref{eq:reverse}.} Write $T:=f(X)$ and $m(T):=\E[U\mid T]=\E[F_X(X)\mid f(X)]$. By the
opening paragraph of the proof, $\gamma=\EVar(U\mid T)=\E\bigl(U-m(T)\bigr)^2$. None of $T$, $m$ and
$\gamma$ depends on $\tau$. Now fix $\tau\ge0$. Since $U=F_X(X)$ is a
function of $X$ and $V_\tau=F_\tau(h(T+\tau N))$ is a function of $(T,N)$,
independence of $X$ and $N$ implies $U\perp V_\tau\mid T$. This conditional
independence means that, once $T$ is known, $V_\tau$ carries no further information about $U$. Therefore,
\begin{equation}\label{eq:ci-mean}
\E[U\mid T,V_\tau]=m(T).
\end{equation}
The conditional law of total variance decomposes the variance of $U$ given
$V_\tau$ by additionally conditioning on $T$:
\[
\Var(U\mid V_\tau)=\E\bigl[\Var(U\mid T,V_\tau)\bigm|V_\tau\bigr]+\Var\bigl(\E[U\mid T,V_\tau]\bigm|V_\tau\bigr).
\]
The first term is the variation of $U$ that remains when $T$ is observed
in addition to $V_\tau$. The second is the variation of the refined
prediction $\E[U\mid T,V_\tau]$ around the coarser prediction
$\E[U\mid V_\tau]$. Taking expectations, and using the tower property in
the first term, gives
\begin{equation}\label{eq:ltv}
\EVar(U\mid V_\tau)=\EVar(U\mid T,V_\tau)+\E\Var\bigl(\E[U\mid T,V_\tau]\bigm|V_\tau\bigr).
\end{equation}
By \eqref{eq:ci-mean}, the first term on the right of \eqref{eq:ltv}
equals $\E\bigl(U-m(T)\bigr)^2=\gamma$, and the second equals
$\E\Var\bigl(m(T)\mid V_\tau\bigr)$, which is nonnegative. Therefore
\[
\EVar(U\mid V_\tau)=\gamma+\E\Var\bigl(m(T)\bigm|V_\tau\bigr)\ \ge\ \gamma.
\]

\emph{Proof of \eqref{eq:forward}.} Define
$I_\tau:=\mathbf 1\{Y'_\tau\le Y_\tau\}$.
Since $(X',N')$ is independent of $(X,N)$,
\[
\E[I_\tau\mid X,N]=F_\tau(Y_\tau)=V_\tau.
\]

Thus, by the conditional Jensen inequality (for convex $\varphi$ and
integrable $Z$, $\varphi(\E[Z\mid X,N])\le\E[\varphi(Z)\mid X,N]$ with
$\varphi(z)=z^2$ and $Z=I_\tau-I_0$),
\[
\varepsilon(\tau):=\E[(V_\tau-V_0)^2]=\E\bigl[(\E[I_\tau-I_0\mid X,N])^2\bigr]\le\E[(I_\tau-I_0)^2]=P(I_\tau\ne I_0).
\]

By continuity of $h$, both $Y_\tau\to Y_0$ and
$Y'_\tau\to Y'_0$ as $\tau\downarrow0$.
Since $Y_0$ and $Y'_0$ are independent and continuously distributed,
they are unequal almost surely. Their ordering is therefore unchanged
for all sufficiently small $\tau$, so $I_\tau\to I_0$ almost surely.
As $(I_\tau-I_0)^2\le1$, dominated convergence gives
$\varepsilon(\tau)\to0$.

At zero noise, $Y_0=h(f(X))$ and $V_0=F_0(Y_0)$.
Since $X=F_X^{-1}(U)$ a.s., we can express $V_0$ directly in terms of $U$:
\[
V_0=g_0(U),\qquad
g_0(u):=F_0\!\left(h\!\left(f\!\left(F_X^{-1}(u)\right)\right)\right).
\]
Here $F_X^{-1}$ recovers the cause from its rank, $h\circ f$ gives
the noiseless response, and $F_0$ maps that response to its rank.

Let $r_\tau(u):=\E[V_\tau\mid U=u]$. Applying \eqref{eq:bestpred} with
predictor $g_0(U)$, and using the uniformity of $U$, gives

\begin{equation}\label{eq:fwd}
\EVar(V_\tau\mid U)+\|r_\tau-g_0\|_2^2
=\E[(V_\tau-g_0(U))^2]
=\varepsilon(\tau)\longrightarrow0.
\end{equation}
Both terms on the left are nonnegative, so each tends to zero.

{\color{black}
\emph{Proof of \eqref{eq:order}.} Let $0<\sigma\le1$, so that
$\frac14\le w_\sigma\le4$ by Lemma~\ref{lem:weights}. By
Lemma~\ref{lem:Dsig}, $\mathcal D_\sigma[r_\tau]\le2\mathcal D_\sigma[g_0]
+2\mathcal D_\sigma[r_\tau-g_0]$ and
$\mathcal D_\sigma[r_\tau-g_0]/\mu_\sigma\le8\|r_\tau-g_0\|_2^2$. Hence, by
\eqref{eq:Q} and \eqref{eq:fwd}, the forward score satisfies
\[
Q_\sigma(U\to V_\tau)
\le4\,\EVar(V_\tau\mid U)+16\|r_\tau-g_0\|_2^2
+2\,\frac{\mathcal D_\sigma[g_0]}{\mu_\sigma}
\le16\,\varepsilon(\tau)+2\,\frac{\mathcal D_\sigma[g_0]}{\mu_\sigma},
\]
while by \eqref{eq:reverse}, $w_\sigma\ge\frac14$ and $\mathcal D_\sigma\ge0$
the reverse score satisfies
\[
Q_\sigma(V_\tau\to U)\ge\tfrac14\,\EVar(U\mid V_\tau)\ge\tfrac14\,\gamma .
\]
Since $g_0$ does not depend on $\tau$, Lemma~\ref{lem:Dto0} gives
$\beta_0\in(0,1]$ with $\mathcal D_\sigma[g_0]/\mu_\sigma\le\gamma/32$ for
$0<\sigma<\beta_0$, and \eqref{eq:fwd} gives $\tau_0>0$ with
$\varepsilon(\tau)\le\gamma/256$ for $0<\tau<\tau_0$. For such
$(\tau,\sigma)$, $Q_\sigma(U\to V_\tau)\le\gamma/8<\gamma/4\le
Q_\sigma(V_\tau\to U)$.
}

\emph{Consistency.} Fix such $(\tau,\sigma)$, and put
$(U_i,V_{\tau,i}):=(F_X(X_i),F_\tau(Y_{\tau,i}))$ for an i.i.d.\ sample
$(X_i,Y_{\tau,i})_{i\ge1}$. These pairs are i.i.d.\ with uniform marginals,
and almost surely their ranks coincide with those of the $X_i$ and
$Y_{\tau,i}$. By Proposition~\ref{prop:pop} and paragraph \emph{Empirical ranks}, the
two rank-based scores converge almost surely to
$Q_\sigma(U\to V_\tau)<Q_\sigma(V_\tau\to U)$. Hence the rank-based rule with bandwidth $\sigma$ selects $X\to Y_\tau$
for all sufficiently large $n$.
\end{proof}

{\color{black}
\section{Causal information at positive bandwidth}\label{app:smoothness}

Recall that $r(u)=\E[V\mid U=u]$ and $s(v)=\E[U\mid V=v]$, and write
\[
\nu_r(u):=\Var(V\mid U=u),
\qquad
\nu_s(v):=\Var(U\mid V=v).
\]
By Proposition~\ref{prop:pop} and \eqref{eq:Q}, the population scores are
\[
\begin{aligned}
Q_\sigma(U\to V)
&=\int_0^1 w_\sigma(u)\nu_r(u)\,du
  +\frac{\mathcal D_\sigma[r]}{\mu_\sigma},\\
Q_\sigma(V\to U)
&=\int_0^1 w_\sigma(v)\nu_s(v)\,dv
  +\frac{\mathcal D_\sigma[s]}{\mu_\sigma}.
\end{aligned}
\]
Lemma~\ref{lem:weights}(ii) gives the first-order expansion of the
degree-weighted conditional-variance terms as $\sigma\downarrow0$.
The following lemma controls the Dirichlet terms. For $g\in L^2[0,1]$
and $0\le h\le1$, define
\[
\omega_g(h):=\int_0^{1-h}\bigl(g(t+h)-g(t)\bigr)^2\,dt.
\]

\begin{lemma}[Rates of the Dirichlet form]\label{lem:rates}
Let $g\in L^2[0,1]$ and $\sigma>0$.
\begin{enumerate}
\item[(a)] The Dirichlet form admits the representation
\[
\mathcal D_\sigma[g]
=\int_0^1 e^{-h^2/(2\sigma^2)}\omega_g(h)\,dh.
\]
\item[(b)] If $g$ is Lipschitz with constant $L_g$, then
\[
\mathcal D_\sigma[g]\le\sqrt{\pi/2}\,L_g^2\sigma^3.
\]
\item[(c)] Suppose that $g$ takes values in $[0,1]$, is locally
absolutely continuous on $(0,1)$, and satisfies
\[
|g'(u)|\le\frac{C}{\sqrt{u(1-u)}}
\qquad\text{for almost every }u\in(0,1).
\]
With $C_3:=\max\{4C^2,4\}$, one has
\[
\omega_g(h)\le C_3h^2\log(e/h),\qquad 0<h\le1,
\]
and
\[
\mathcal D_\sigma[g]\le3C_3\sigma^3\log(1/\sigma),
\qquad 0<\sigma\le e^{-1}.
\]
\item[(d)] Suppose that $g$ is locally absolutely continuous on
$(0,1)$ and that, for some $c>0$ and $v_1\in(0,1)$,
\[
g'(u)\ge\frac{c}{\sqrt u}
\qquad\text{for almost every }u\in(0,v_1).
\]
Then
\[
\omega_g(h)\ge c^2h^2\log(v_1/h),\qquad 0<h<v_1,
\]
and
\[
\mathcal D_\sigma[g]\ge\frac{c^2}{6\sqrt e}\,
\sigma^3\log(1/\sigma),\qquad 0<\sigma<v_1^2.
\]
\end{enumerate}
\end{lemma}

\begin{proof}
\emph{(a)} By symmetry of the integrand in the definition of
$\mathcal D_\sigma$, the factor $1/2$ is removed by restricting the
double integral to $0<t<u<1$. Setting $h=u-t$ and applying Tonelli's
theorem gives
\[
\begin{aligned}
\mathcal D_\sigma[g]
&=\int_0^1\int_0^u e^{-(u-t)^2/(2\sigma^2)}
  \bigl(g(u)-g(t)\bigr)^2\,dt\,du\\
&=\int_0^1 e^{-h^2/(2\sigma^2)}
  \int_0^{1-h}\bigl(g(t+h)-g(t)\bigr)^2\,dt\,dh.
\end{aligned}
\]

\emph{(b)} The Lipschitz property gives
$\omega_g(h)\le L_g^2h^2(1-h)\le L_g^2h^2$. The result follows from (a)
and
\[
\int_0^\infty e^{-h^2/(2\sigma^2)}h^2\,dh
=\sqrt{\pi/2}\,\sigma^3.
\]

\emph{(c)} For $0<t<t+h<1$, local absolute continuity and the
Cauchy--Schwarz inequality give
\[
\bigl(g(t+h)-g(t)\bigr)^2
\le h\int_t^{t+h}|g'(x)|^2\,dx.
\]
For each $x\in(0,1)$, the set of $t\in[0,1-h]$ satisfying
$t\le x\le t+h$ has length at most $\min\{h,x,1-x\}$. Integration in
$t$, followed by Tonelli's theorem, therefore yields
\[
\begin{aligned}
\omega_g(h)
&\le h\int_0^1 |g'(x)|^2\min\{h,x,1-x\}\,dx\\
&\le C^2h\int_0^1
  \frac{\min\{h,x,1-x\}}{x(1-x)}\,dx.
\end{aligned}
\]
For $0<h\le1/2$, the last integrand is symmetric about $1/2$, and
\[
\int_0^{1/2}\frac{\min\{h,x\}}{x(1-x)}\,dx
\le2\int_0^{1/2}\frac{\min\{h,x\}}{x}\,dx
=2h\left(1+\log\frac{1}{2h}\right).
\]
Hence $\omega_g(h)\le4C^2h^2\log(e/h)$ for $h\le1/2$.
For $1/2<h\le1$, the range restriction gives
$\omega_g(h)\le1-h\le1\le4h^2\log(e/h)$.
Combining these bounds proves the stated estimate for $\omega_g$.
By (a), with the substitution $h=\sigma z$,
\[
\begin{aligned}
\mathcal D_\sigma[g]
&\le C_3\sigma^3\int_0^{1/\sigma}e^{-z^2/2}z^2
  \left(1+\log\frac1\sigma-\log z\right)\,dz\\
&\le C_3\sigma^3\left[
  \sqrt{\pi/2}\left(1+\log\frac1\sigma\right)+\frac19\right].
\end{aligned}
\]
Here we used
\[
-\int_0^{1/\sigma}e^{-z^2/2}z^2\log z\,dz
\le\int_0^1 z^2|\log z|\,dz=\frac19.
\]
For $0<\sigma\le e^{-1}$, the bracket is at most
$3\log(1/\sigma)$, which proves the claim.

\emph{(d)} For $0<h<v_1$ and $0<t<v_1-h$,
\[
g(t+h)-g(t)
\ge c\int_t^{t+h}x^{-1/2}\,dx
=\frac{2ch}{\sqrt{t+h}+\sqrt t}
\ge\frac{ch}{\sqrt{t+h}}>0.
\]
Consequently,
\[
\omega_g(h)
\ge c^2h^2\int_0^{v_1-h}\frac{1}{t+h}\,dt
=c^2h^2\log(v_1/h).
\]
If $0<\sigma<v_1^2<v_1$, restricting the integral in (a) to
$0<h<\sigma$ gives
\[
\begin{aligned}
\mathcal D_\sigma[g]
&\ge e^{-1/2}c^2\log(v_1/\sigma)\int_0^\sigma h^2\,dh\\
&=\frac{c^2}{3\sqrt e}\,\sigma^3\log(v_1/\sigma).
\end{aligned}
\]
Finally, $\sigma<v_1^2$ implies
$\log(v_1/\sigma)\ge\tfrac12\log(1/\sigma)$.
\end{proof}

Since $\mu_\sigma=\sigma\sqrt{2\pi}\,\bar\psi_\sigma\asymp\sigma$
as $\sigma\downarrow0$, a Lipschitz function satisfies
$\mathcal D_\sigma[g]/\mu_\sigma=O(\sigma^2)$.
This is of smaller order than the $O(\sigma)$ correction in
Lemma~\ref{lem:weights}(ii). An inverse-square-root singularity of the
derivative can introduce an additional factor $\log(1/\sigma)$.
The next result establishes this behavior for the reverse regression
function while showing that the forward regression function is Lipschitz.

\begingroup
\renewcommand{\theproposition}{\ref{prop:hetero}}%
\ifdefined\theHproposition
  \renewcommand{\theHproposition}{prop-hetero-restated}%
\fi
\begin{proposition}[Bounded location--scale mechanisms]
Let $Y=f(X)+g(X)N$, where $X$ and $N$ are independent and supported on
$[-1,1]$ and $[-\eta,\eta]$, $\eta>0$, with continuously differentiable
densities that are strictly positive on these closed intervals. Let
$f,g\in C^2([-1,1])$ with $g>0$ and $f'-\eta|g'|\ge\lambda>0$ on
$[-1,1]$, and put $E_r=\EVar(V\mid U)$ and $E_s=\EVar(U\mid V)$. Then:
\begin{enumerate}
\item[(i)] There exist $0<c_1\le C_1<\infty$ and $\sigma_1\in(0,1)$ such
that, for $0<\sigma<\sigma_1$, $\mathcal D_\sigma[r]\le C_1\sigma^3$
and $c_1\sigma^3\log(1/\sigma)\le\mathcal D_\sigma[s]\le
C_1\sigma^3\log(1/\sigma)$.
\item[(ii)] $\nu_r$ has a continuous version on $[0,1]$ with
$\nu_r(0),\nu_r(1)>0$, $\nu_s$ has a version with
$\nu_s(0^+)=\nu_s(1^-)=0$, and, as $\sigma\downarrow0$,
\[
Q_\sigma(V\!\to\!U)-Q_\sigma(U\!\to\!V)
=\bigl(1+\sigma\sqrt{2/\pi}\bigr)(E_s-E_r)
+\tfrac{\sigma}{\sqrt{2\pi}}\bigl(\nu_r(0)+\nu_r(1)\bigr)+o(\sigma).
\]
\end{enumerate}
Consequently, if $E_r\le E_s$, then
$Q_\sigma(U\!\to\!V)<Q_\sigma(V\!\to\!U)$ for all sufficiently small
$\sigma>0$. The monotone additive model $Y=f(X)+N$ with
$f'\ge\lambda>0$ is the special case $g\equiv1$.
\end{proposition}
\addtocounter{proposition}{-1}
\endgroup

\begin{proof}
We first show that the forward regression $r(u)$ is Lipschitz. For the
reverse regression $s(v)$, the main point is that $p_Y(y)$ vanishes
linearly at the endpoints of its support. Passing to the uniform
coordinate $v=F_Y(y)$ then produces an inverse-square-root singularity
in $s'(v)$ as $v$ approaches $0$ or $1$.
Lemma~\ref{lem:rates} converts these regularity properties into (i),
and Lemma~\ref{lem:weights}(ii), together with (i), gives the score
expansion in (ii).
We write $\alpha\asymp\beta$ when their ratio is bounded above and
below by positive constants in the stated regime.

\emph{Step 1: support and density.}
For $n\in[-\eta,\eta]$, put
\[
\ell_n(x):=f(x)+ng(x),\qquad
\ell_-(x):=\ell_{-\eta}(x),\qquad
\ell_+(x):=\ell_\eta(x),
\]
and let $\Lambda:=\|f'\|_\infty+\eta\|g'\|_\infty$, with norms on
$[-1,1]$. The assumptions give
{\color{black}
\begin{equation}\label{eq:hetero-slope-bounds}
\lambda\le\ell_n'(x)\le\Lambda,\qquad
\ell_-(x)<\ell_+(x).
\end{equation}
}
Thus every $\ell_n$ is strictly increasing. The joint support of
$(X,Y)$ and the support of $Y$ are, respectively,
\[
R:=\{(x,y):-1\le x\le1,\ \ell_-(x)\le y\le\ell_+(x)\},
\qquad
[A,B]:=[\ell_-(-1),\ell_+(1)].
\]
On $R$, the joint density is
\[
p_{X,Y}(x,y)
=\frac{p_X(x)}{g(x)}
  p_N\!\left(\frac{y-f(x)}{g(x)}\right).
\]

Extend $p_X,f,g$ and $p_N$ beyond their respective intervals by the
first-order Taylor polynomials at the endpoints, retaining the same
notation, and choose $\delta>0$ such that the extension of $g$ is
positive on $(-1-\delta,1+\delta)$.
On $O:=(-1-\delta,1+\delta)\times\R$, this density formula
using these extensions defines $\bar p\in C^1(O)$ with
$\bar p=p_{X,Y}$ on $R$, and positivity and compactness give constants
{\color{black}
\begin{equation}\label{eq:hetero-joint-bounds}
0<p_{\min}\le p_{X,Y}\le p_{\max}<\infty,
\qquad
|\partial_x\bar p|\le K_x,\quad |\partial_y\bar p|\le K_y
\quad\text{on }R.
\end{equation}
}
On $R$, partial derivatives of $p_{X,Y}$ are understood as those of
$\bar p$, so they are defined and continuous on all of $R$, including its
boundary. For $y\in[A,B]$, define the endpoints of the conditional support by
\[
\begin{aligned}
x_-(y)&:=\min\{x\in[-1,1]:\ell_+(x)\ge y\},\\
x_+(y)&:=\max\{x\in[-1,1]:\ell_-(x)\le y\}.
\end{aligned}
\]
These functions are $1/\lambda$-Lipschitz.
{\color{black}
Strict monotonicity gives
\[
\{x\in[-1,1]:\ell_+(x)\ge y\}=[x_-(y),1],\qquad
\{x\in[-1,1]:\ell_-(x)\le y\}=[-1,x_+(y)],
\]
so the section of $R$ at height $y$ is $[x_-(y),x_+(y)]$.
For $A<y<B$, put $x_0=x_-(y)<1$ (since $\ell_+(1)=B>y$).
If $x_0=-1$, then $\ell_-(x_0)=A<y$. Otherwise $\ell_+(x_0)=y$, hence again
$\ell_-(x_0)<y$. In both cases, continuity of $\ell_-$ yields $x_+(y)>x_0$.
}
Consequently,
\[
p_Y(y)=\int_{x_-(y)}^{x_+(y)}p_{X,Y}(x,y)\,dx
\]
is continuous and bounded on $[A,B]$, and strictly positive on
$(A,B)$. In particular, $F_Y$ is continuously differentiable and
strictly increasing on $[A,B]$.

\smallskip
\noindent\textbf{What this step establishes.}
\begin{itemize}
\item The joint support is $R$, the effect support is $[A,B]$, and
the conditional support of $X$ given $Y=y$ is $[x_-(y),x_+(y)]$.
\item On $R$, $p_{X,Y}$ is bounded above and away from zero, with
bounded first partial derivatives.
\item The marginal density $p_Y$ is continuous and bounded, and is
positive on $(A,B)$. Hence $F_Y$ is continuously differentiable and
strictly increasing on $[A,B]$.
\end{itemize}

\emph{Step 2: forward regression and conditional variance.}
The map $F_X:[-1,1]\to[0,1]$ is a continuously differentiable
bijection, and $p_X$ is bounded away from zero. Since $U=F_X(X)$
and $V=F_Y(Y)$, invertibility of $F_X$ allows us to condition on
$X=x$ in place of $U=F_X(x)$. The definition of conditional
expectation gives
\[
\begin{aligned}
r(F_X(x))
&=\int_A^B F_Y(y)\,p_{Y\mid X}(y\mid x)\,dy
&=\int_{-\eta}^{\eta}F_Y(\ell_n(x))\,p_{N\mid X}(n\mid x)\,dn
&=\int_{-\eta}^{\eta}F_Y(\ell_n(x))\,p_N(n)\,dn.
\end{aligned}
\]
Here $p_{Y\mid X}$ and $p_{N\mid X}$ denote conditional densities.
The penultimate equality uses $Y=f(x)+g(x)N=\ell_N(x)$ when $X=x$, and
the last uses independence of the cause $X$ and the noise $N$, which
gives $p_{N\mid X}(n\mid x)=p_N(n)$. We use the final integral to
define the endpoint values at $x=\pm1$ as well.
Differentiating under the integral and applying the chain rule,
\[
r'(F_X(x))
=\frac{1}{p_X(x)}
  \int_{-\eta}^{\eta}p_Y(\ell_n(x))\ell_n'(x)p_N(n)\,dn.
\]
Every factor is uniformly bounded, so $r$ is Lipschitz on $[0,1]$.
Similarly,
\[
\nu_r(F_X(x))
=\int_{-\eta}^{\eta}F_Y(\ell_n(x))^2p_N(n)\,dn
 -r(F_X(x))^2
\]
defines a continuous version of $\nu_r$. For each fixed $x$, the
map $n\mapsto F_Y(\ell_n(x))$ is strictly increasing, since $g(x)>0$
and $F_Y$ is strictly increasing. As $N$ is nondegenerate, this
conditional variance is positive, including at $x=-1$ and $x=1$.
Thus $\nu_r(0),\nu_r(1)>0$.

\smallskip
\noindent\textbf{What this step establishes.}
\begin{itemize}
\item The forward regression $r$ is Lipschitz on $[0,1]$, providing
the regularity needed for the causal Dirichlet-energy bound.
\item The forward conditional variance $\nu_r$ has a continuous
version on $[0,1]$ with strictly positive endpoint values
$\nu_r(0),\nu_r(1)>0$.
\end{itemize}

\emph{Step 3: reverse regression near the lower endpoint.}
Keep the notation $s$ for the regression on the uniform scale and
write
\[
\begin{aligned}
s_Y(y):=\E[U\mid Y=y]
&=\int_{x_-(y)}^{x_+(y)}
  F_X(x)p_{X\mid Y}(x\mid y)\,dx
&=\frac{\displaystyle\int_{x_-(y)}^{x_+(y)}
  F_X(x)p_{X,Y}(x,y)\,dx}{p_Y(y)},
\qquad A<y<B.
\end{aligned}
\]

Choose $\varepsilon>0$ so that
$A+\varepsilon<\min\{\ell_+(-1),\ell_-(1)\}$.
For $A<y<A+\varepsilon$, the conditional support is
$[-1,\ell_-^{-1}(y)]$. Its length
\[
L(y):=\ell_-^{-1}(y)+1
\]
satisfies
\[
L(y)\asymp y-A,\qquad
L'(y)=\frac{1}{\ell_-'(\ell_-^{-1}(y))}
\longrightarrow\frac{1}{\ell_-'(-1)}>0.
\]
{\color{black}
Integrating the density bounds in~\eqref{eq:hetero-joint-bounds}
over the conditional support $[-1,-1+L(y)]$ gives
\[
p_{\min}L(y)
\le p_Y(y)=\int_{-1}^{-1+L(y)}p_{X,Y}(x,y)\,dx
\le p_{\max}L(y).
\]
Since $y-A=\ell_-(-1+L(y))-\ell_-(-1)$, the mean value theorem
and~\eqref{eq:hetero-slope-bounds} give
\[
\frac{y-A}{\Lambda}\le L(y)\le\frac{y-A}{\lambda}.
\]
Thus $p_Y(y)\asymp L(y)\asymp y-A$. Combining these bounds and
integrating from $A$ to $y$ yields
\[
\frac{p_{\min}}{2\Lambda}(y-A)^2
\le F_Y(y)=\int_A^y p_Y(z)\,dz
\le\frac{p_{\max}}{2\lambda}(y-A)^2,
\]
so $F_Y(y)\asymp(y-A)^2$.
}

We next show that $s_Y'(y)\to p_X(-1)/\bigl(2\ell_-'(-1)\bigr)>0$ as $y\downarrow A$.
Rescale the shrinking conditional support by $x=-1+L(y)\xi$ and
define the density
\[
\rho_y(\xi):=
\frac{p_{X,Y}(-1+L(y)\xi,y)}
     {\displaystyle\int_0^1p_{X,Y}(-1+L(y)t,y)\,dt},
\qquad 0\le\xi\le1.
\]
This is the conditional density of $(X+1)/L(y)$ given $Y=y$.
Since
\[
F_X(-1+L(y)\xi)
=L(y)\int_0^\xi p_X(-1+L(y)t)\,dt,
\]
we have
\[
\begin{aligned}
s_Y(y)&=L(y)H(y),\\
H(y)&:=
\int_0^1\left[\int_0^\xi p_X(-1+L(y)t)\,dt\right]
          \rho_y(\xi)\,d\xi.
\end{aligned}
\]
As $y\downarrow A$, continuity and strict positivity at the corner
$(-1,A)$ give $\rho_y\to1$ uniformly on $[0,1]$. Hence
\[
H(y)\xrightarrow[\;y\downarrow A\;]{}
p_X(-1)\int_0^1\xi\,d\xi=\frac{p_X(-1)}{2}.
\]

On $J:=(A,A+\varepsilon)$, the inverse function theorem gives
$L\in C^1(J)$ and $1/\Lambda\le L'\le1/\lambda$. Write
\[
\begin{aligned}
a(\xi,y)&:=p_{X,Y}(-1+L(y)\xi,y),
&Z(y)&:=\int_0^1a(t,y)\,dt,\\
G(\xi,y)&:=\int_0^\xi p_X(-1+L(y)t)\,dt,
&\rho_y(\xi)&=\frac{a(\xi,y)}{Z(y)},
\end{aligned}
\qquad H(y)=\int_0^1G(\xi,y)\rho_y(\xi)\,d\xi.
\]
The chain rule gives, uniformly for $\xi\in[0,1]$ and $y\in J$,
\[
\begin{aligned}
\partial_y a(\xi,y)
&=\partial_xp_{X,Y}(-1+L(y)\xi,y)L'(y)\xi
  +\partial_yp_{X,Y}(-1+L(y)\xi,y),\\
|\partial_y a(\xi,y)|&\le K_x/\lambda+K_y,\\
\partial_y G(\xi,y)
&=L'(y)\int_0^\xi p_X'(-1+L(y)t)t\,dt,
\qquad |\partial_y G(\xi,y)|\le\frac{\|p_X'\|_\infty}{2\lambda}.
\end{aligned}
\]
These uniform bounds justify differentiation under the integrals
by the mean value theorem and dominated convergence. The quotient
rule, $p_{\min}\le a,Z\le p_{\max}$, $\int_0^1\rho_y=1$, and
$0\le G\le\|p_X\|_\infty$ then give
\[
\begin{aligned}
|\partial_y\rho_y(\xi)|
&\le (K_x/\lambda+K_y)
  \left(\frac1{p_{\min}}+\frac{p_{\max}}{p_{\min}^2}\right),\\
|H'(y)|
&\le\frac{\|p_X'\|_\infty}{2\lambda}
 +\|p_X\|_\infty(K_x/\lambda+K_y)
  \left(\frac1{p_{\min}}+\frac{p_{\max}}{p_{\min}^2}\right).
\end{aligned}
\]
In particular, $L(y)H'(y)\to0$ as $y\downarrow A$. It follows that
\[
s_Y'(y)=L'(y)H(y)+L(y)H'(y)
\xrightarrow[\;y\downarrow A\;]{}
\frac{p_X(-1)}{2\ell_-'(-1)}>0.
\]

Thus $s_Y'$ is bounded and bounded away from zero near $A$.  Also, for $y\in J$, the conditional support $[-1,-1+L(y)]$
of $X$ given $Y=y$ shrinks to $\{-1\}$ as $y\downarrow A$.
Since $U=F_X(X)$ and $F_X$ is nondecreasing, we have
$0\le U\le F_X(-1+L(y))$ conditional on $Y=y$. Thus,
\[
\Var(U\mid Y=y)=\E[U^2\mid Y=y]-\E[U\mid Y=y]^2 \le \E[U^2\mid Y=y]\le F_X(-1+L(y))^2\xrightarrow[\;y\downarrow A\;]{}0.
\]
\smallskip
\noindent\textbf{What this step establishes.}
\begin{itemize}
\item As $y\downarrow A$, the conditional support shrinks with
length $L(y)\asymp y-A$, giving $p_Y(y)\asymp y-A$ and
$F_Y(y)\asymp(y-A)^2$.
\item The reverse regression on the original effect scale satisfies
$s_Y'(y)\to p_X(-1)/(2\ell_-'(-1))>0$.
\item The reverse conditional variance satisfies
$\Var(U\mid Y=y)\to0$ as $y\downarrow A$.
\end{itemize}

\emph{Step 4: reverse regression on the uniform scale.}
The upper endpoint follows by applying Step 3 to
\[
\widetilde X=-X,\qquad \widetilde N=-N,\qquad
\widetilde Y=-Y,\qquad
\widetilde f(x)=-f(-x),\qquad \widetilde g(x)=g(-x).
\]
This model satisfies the same assumptions, because
$\widetilde f'(x)-\eta|\widetilde g'(x)|
=f'(-x)-\eta|g'(-x)|\ge\lambda$.
Its uniform variables are $\widetilde U=1-U$ and
$\widetilde V=1-V$. In particular,
$s_Y(y)=1-\widetilde s_Y(-y)$, so as $y\uparrow B$,
\[
p_Y(y)\asymp B-y,\qquad
1-F_Y(y)\asymp(B-y)^2,
\]
$s_Y'$ is bounded, and $\Var(U\mid Y=y)\to0$.

On any compact subinterval of $(A,B)$, $p_Y$ is bounded away from
zero. The numerator and denominator defining $s_Y$ are Lipschitz
there: their integration endpoints are Lipschitz, their integrands
are bounded, and their derivatives in $y$ are bounded wherever the
integrands are evaluated. Thus $s_Y$ is locally Lipschitz on $(A,B)$.

Since $F_Y$ is strictly increasing, conditioning on $V=v$ is the same as
conditioning on $Y=F_Y^{-1}(v)$. We may therefore take
\[
s(v)=s_Y(F_Y^{-1}(v)),\qquad 0<v<1,
\]
which is locally Lipschitz on $(0,1)$.

{\color{black}Since $p_Y>0$ on $(A,B)$, $F_Y^{-1}$ is a $C^1$ diffeomorphism of $(0,1)$
onto $(A,B)$. Hence $s$ is locally absolutely continuous, and the chain
rule below holds almost everywhere.}
Therefore,
\[
s'(v)=\frac{s_Y'(y)}{p_Y(y)},\qquad y=F_Y^{-1}(v).
\]
Near $v=0$, Step 3 gives $y-A\asymp\sqrt v$ and
$p_Y(y)\asymp\sqrt v$, while $s_Y'$ is bounded and strictly
positive. The reflected argument gives the corresponding upper
bound near $v=1$. Combining these endpoint estimates with the
local Lipschitz property in the interior, there are
$C_s,c_s>0$ and $v_1\in(0,1)$ such that
\[
|s'(v)|\le\frac{C_s}{\sqrt{v(1-v)}}
\quad\text{for almost every }v\in(0,1),
\qquad
s'(v)\ge\frac{c_s}{\sqrt v}
\quad\text{for almost every }v\in(0,v_1).
\]
The conditional variances in Steps 3 and 4 also give a version $\nu_s(v)=\Var(U\mid Y=F_Y^{-1}(v))$ with $\nu_s(0^+)=\nu_s(1^-)=0$.

\smallskip
\noindent\textbf{What this step establishes.}
\begin{itemize}
\item The reverse regression $s$ is locally Lipschitz on $(0,1)$.
Its derivative satisfies the global bound
$|s'(v)|\le C_s/\sqrt{v(1-v)}$ almost everywhere and the lower
bound $s'(v)\ge c_s/\sqrt v$ almost everywhere near $v=0$.
\item The reverse conditional variance has a version with
$\nu_s(0^+)=\nu_s(1^-)=0$.
\end{itemize}

\emph{Step 5: Dirichlet rates and score ordering.}
Let $L_r$ be a Lipschitz constant of $r$ from Step 2.
Lemma~\ref{lem:rates}(b) gives
\[
\mathcal D_\sigma[r]\le\sqrt{\pi/2}\,L_r^2\sigma^3.
\]
By Step 4, $s$ is $[0,1]$-valued and locally absolutely continuous,
with the derivative bounds stated there. Lemma~\ref{lem:rates}(c) with
$C=C_s$ gives the upper bound for $\sigma\le e^{-1}$, and part (d) with
$c=c_s$ gives the lower bound for $\sigma<v_1^2$, so (i) holds with
\[
\begin{aligned}
\sigma_1&:=\min\{e^{-1},v_1^2\},
\qquad c_1:=\frac{c_s^2}{6\sqrt e},\\
C_1&:=\max\left\{c_1,\sqrt{\pi/2}\,L_r^2,
                     3\max\{4C_s^2,4\}\right\}.
\end{aligned}
\]

For (ii), recall
\[
E_r=\int_0^1\nu_r(u)\,du,\qquad
E_s=\int_0^1\nu_s(v)\,dv.
\]
Both conditional variances are bounded, and their endpoint limits are
$\nu_r(0),\nu_r(1)>0$ (Step 2) and $\nu_s(0^+)=\nu_s(1^-)=0$ (Step 4).
Applying Lemma~\ref{lem:weights}(ii) gives
\[
\begin{aligned}
\int_0^1 w_\sigma(u)\nu_r(u)\,du
&=E_r+\frac{\sigma}{\sqrt{2\pi}}
  \bigl(2E_r-\nu_r(0)-\nu_r(1)\bigr)+o(\sigma),\\
\int_0^1 w_\sigma(v)\nu_s(v)\,dv
&=E_s+\frac{2\sigma}{\sqrt{2\pi}}E_s+o(\sigma).
\end{aligned}
\]
Since $\mu_\sigma\asymp\sigma$, part (i) also gives
\[
\frac{\mathcal D_\sigma[r]}{\mu_\sigma}
=O(\sigma^2),\qquad
\frac{\mathcal D_\sigma[s]}{\mu_\sigma}
=O\!\left(\sigma^2\log(1/\sigma)\right)=o(\sigma).
\]
Subtracting the two population scores gives
\[
\begin{aligned}
Q_\sigma(V\!\to\!U)-Q_\sigma(U\!\to\!V)
&=\bigl(1+\sigma\sqrt{2/\pi}\bigr)(E_s-E_r)\\
&\quad+\frac{\sigma}{\sqrt{2\pi}}
  \bigl(\nu_r(0)+\nu_r(1)\bigr)+o(\sigma).
\end{aligned}
\]
If $E_s>E_r$, this difference tends to the positive constant
$E_s-E_r$. If $E_s=E_r$, its coefficient of $\sigma$ is strictly
positive because $\nu_r(0),\nu_r(1)>0$. Thus the forward score is
smaller for all sufficiently small $\sigma>0$ whenever $E_r\le E_s$.
The additive case follows by taking $g\equiv1$.

\end{proof}

\section{Efficient computation}\label{app:fast}

Throughout this section we assume that, with ties
broken at random, the normalized rank vectors $u,v$ of $X,Y$ are
permutations of the grid $\{1,\dots,n\}/(n+1)$ (Section~\ref{sec:fft}). Let $v^U\in\R^n$ be $v$ reordered by increasing $U$-rank, so that $v^U_k$ is the value of $v$ at the
observation with $U$-rank $k$, and define $u^V$ analogously by reordering
$u$ by increasing $V$-rank.

\paragraph{Kernel in rank units.}
Grid points with ranks $k$ and $l$ are $|k-l|/(n+1)$ apart, so the
median-based bandwidth of Section~\ref{sec:method} satisfies
\[
\widehat\sigma:=\widehat\sigma(u;m)=\widehat\sigma(v;m)=\frac{h}{n+1},\qquad
h:=m\cdot\operatorname{median}_{k<l}|k-l|,
\]
and the kernel weight between the observations with $U$-ranks $k$ and
$l$ is  $\kappa(k-l):=\exp\bigl[-\tfrac{(k-l)^2}{2h^2}\bigr]$.
Hence $W^U$, with the observations indexed by $U$-rank, and $W^V$, with the
observations indexed by $V$-rank, both equal the symmetric Toeplitz matrix
\[
W^\kappa:=\bigl(\kappa(k-l)\bigr)_{k,l=1}^n,
\qquad
d:=W^\kappa\mathbf 1,\qquad d_k=\sum_{l=1}^n\kappa(k-l),
\]
and $d^U$, $d^V$ both equal $d$.

Since exactly $n-j$ pairs have lag
$j\in\{1,\dots,n-1\}$, the median lag, and hence $h$, is found in $O(n)$
time by accumulating these counts. These counts do not depend on the sample, so $\widehat\sigma$ depends only on $n$ and $m$ for every $n$, and
$\widehat\sigma\to m(1-1/\sqrt2)$ as $n\to\infty$.

\begingroup
\renewcommand{\theproposition}{\ref{prop:fast_computation}}%
\ifdefined\theHproposition
  \renewcommand{\theHproposition}{prop-fast-restated}%
\fi
\begin{proposition}[Efficient computation]
Under marginal equalization, with ties broken at random,
$M_n(U)=M_n(V)=\sum_k d_k$ and
\[
\begin{aligned}
E_n(U\!\to\!V)&=\textstyle\sum_k d_k(v^U_k)^2
 -\sum_{k,l}\kappa(k-l)\,v^U_kv^U_l,\\
E_n(V\!\to\!U)&=\textstyle\sum_k d_k(u^V_k)^2
 -\sum_{k,l}\kappa(k-l)\,u^V_ku^V_l,
\end{aligned}
\]
so the GDECI rule selects $X\to Y$ exactly when
$E_n(U\!\to\!V)<E_n(V\!\to\!U)$. Each double sum equals $z^\top W^\kappa z$
for $z\in\{v^U,u^V\}$, where $W^\kappa z$ is a discrete convolution of $z$ with
$(\kappa(j))_{|j|<n}$, so both energies can be evaluated by fast Fourier
transforms in $O(n\log n)$ time and $O(n)$ memory.
\end{proposition}
\addtocounter{proposition}{-1}
\endgroup

\begin{proof}[Proof of Proposition~\ref{prop:fast_computation}]
\emph{Energies and edge masses.} Permuting the rows and columns of $W^U$ and
the entries of $v$ simultaneously leaves $v^\top L^Uv$ and
$\sum_{i,j}W^U_{ij}$ unchanged. Ordering the observations by $U$-rank thus
gives
\[
M_n(U)=\mathbf 1^\top d,
\qquad
E_n(U\!\to\!V)=(v^U)^\top\bigl(\operatorname{diag}(d)-W^\kappa\bigr)v^U
=\textstyle\sum_k d_k(v^U_k)^2-(v^U)^\top W^\kappa v^U,
\]
and ordering by $V$-rank gives the same with $u^V$ in place of $v^U$. In
particular $M_n(V)=M_n(U)$, so comparing $\Delta_n$ reduces to comparing
$E_n$.

\emph{Cost.}
For $k=1,\dots,n$, $(W^\kappa z)_k=\sum_l\kappa(k-l)z_l=(\kappa*z)_k$ is the
convolution of $z$ with the lag sequence $(\kappa(j))_{|j|<n}$. Let $P$ be a power of two with $2n-1\le P<4n$ (the upper bound gives $P=O(n)$), and let $C$ be the circulant matrix of order $P$ with first row
\[
c:=\bigl(\kappa(0),\dots,\kappa(n-1),0,\dots,0,\kappa(n-1),\dots,\kappa(1)\bigr)
\in\R^P,
\]
each further row being the previous one shifted cyclically by one position to
the right. Since $P\ge2n-1$, the two blocks of $\kappa$-values in $c$ do not
overlap, and the upper-left $n\times n$ block of $C$ is $W^\kappa$. Hence $W^\kappa z$ consists of the first $n$
entries of $C\bar z$, where $\bar z\in\R^P$ is $z$ padded with zeros. The matrix $C$ is never formed, since $C\bar z$ is the circular convolution of
$c$ with $\bar z$, which the fast Fourier transform computes in
$O(P\log P)=O(n\log n)$ time. Two such products give $W^\kappa v^U$ and
$W^\kappa u^V$, and the degrees $d_k=\sum_{j=k-n}^{k-1}\kappa(j)$ follow from
cumulative sums of $\kappa$ in $O(n)$ time. With the two sorts for the rank
transforms and a fixed number of inner products, each costing $O(n)$, the
total cost is $O(n\log n)$. Apart from FFT work arrays of length $P$, only
$c$, $d$ and the two signals are stored, giving $O(n)$ memory.
\end{proof}

\section{Auxiliary tables and figures}\label{app:auxilairy_info}
\begin{figure*}[h]
  \centering
  \includegraphics[width=0.7\textwidth]{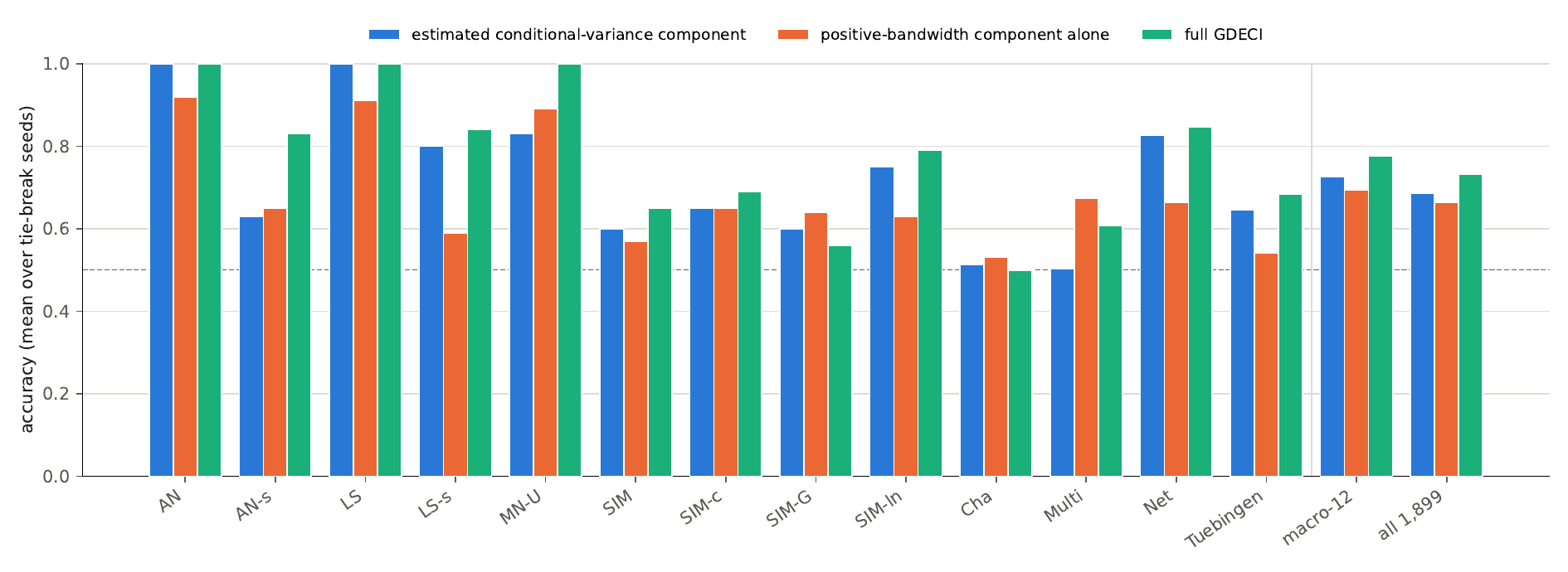}\hfill
  \includegraphics[width=0.3\textwidth]{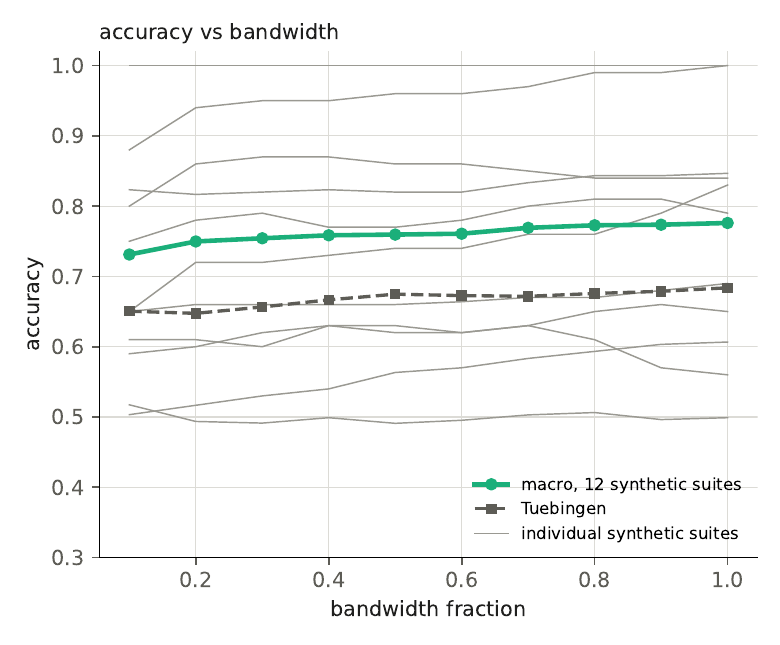}
  \caption{Positive-bandwidth ablation over all $1{,}899$ pairs. (a) Accuracy of
  the estimated conditional-variance component (blue), the positive-bandwidth component alone
  (orange), and the full GDECI score (green), by suite, macro over the 12
  synthetic suites, and pooled over all pairs. (b) Accuracy of the full score
  at bandwidth $s\,\sigma_0$, $s \in \{0.1, 0.2, \dots, 1\}$: macro mean over
  the 12 synthetic suites (solid green), Tuebingen (dashed) and individual
  synthetic suites (grey).}
  \label{fig:ablation}
\end{figure*}
\begin{table}[h]
  \centering
  \caption{Robustness of the small-bandwidth decomposition to the fitting
  window. In each row, $B + a\,\sigma$ is fitted on the $k$ smallest
  bandwidths ($s \le k/10$) as per Section~\ref{sec:ablation}. Corrected, misdirected and net counts compare the $B$-alone and full-score decisions. Accuracies are pooled over all $1{,}899$ pairs and averaged over
  ten tie-breaks. Decision counts and McNemar $p$ use a single tie-break. The
  main analysis uses $k = 5$ (bold).}
  \label{tab:fitwindow}
  \small
  \begin{tabular}{lcccccc}
    \toprule
    window & $k$ & $B$ alone & $\sigma>0$ part alone
      & corrected & misdirected & net (McNemar $p$) \\
    \midrule
    $s \le 0.3$ & 3  & 0.677 & 0.668 & 207 & 107 & $+100$ ($1.8\times10^{-8}$) \\
    $s \le 0.4$ & 4  & 0.682 & 0.666 & 195 & 103 & $+92$ ($1.1\times10^{-7}$) \\
    $\mathbf{s \le 0.5}$ & \textbf{5} & \textbf{0.686} & \textbf{0.663}
      & \textbf{185} & \textbf{101} & $\mathbf{+84}$ ($\mathbf{7.7\times10^{-7}}$) \\
    $s \le 0.6$ & 6  & 0.690 & 0.654 & 175 & 97  & $+78$ ($2.6\times10^{-6}$) \\
    $s \le 0.7$ & 7  & 0.691 & 0.647 & 170 & 91  & $+79$ ($1.2\times10^{-6}$) \\
    $s \le 0.8$ & 8  & 0.693 & 0.640 & 163 & 87  & $+76$ ($1.8\times10^{-6}$) \\
    $s \le 0.9$ & 9  & 0.695 & 0.642 & 158 & 85  & $+73$ ($3.3\times10^{-6}$) \\
    \bottomrule
  \end{tabular}
\end{table}